\documentclass{article}

\usepackage{microtype}
\usepackage{graphicx}
\usepackage{subfigure}
\usepackage{booktabs}
\usepackage{enumerate}
\usepackage{bbm}

\usepackage{hyperref}
\usepackage{amsmath}
\usepackage{amssymb}
\usepackage{amsthm}
\usepackage{authblk}
\usepackage{natbib}
\usepackage{algorithmic}
\usepackage{algorithm}

\usepackage{bbold}
\usepackage{xfrac}
\usepackage{mathtools}
\usepackage{physics}
\usepackage{wrapfig}
\usepackage{nicefrac}

\setcitestyle{authoryear,round,citesep={;},aysep={,},yysep={;}}

\usepackage[margin=1in]{geometry}

\usepackage[capitalize,noabbrev]{cleveref}

\def\max{{\rm max}}

\newcommand{\E}{\mathbb{E}}
\newcommand{\R}{\mathbb{R}}
\newcommand{\Ea}[1]{\E\left[#1\right]}
\newcommand{\Eb}[2]{\E_{#1}\left[#2\right]}
\usepackage{amsthm}

\newtheorem{theorem}{Theorem}[section]

\newtheorem{lemma}[theorem]{Lemma}
\newtheorem{assumption}{Assumption}
\newtheorem{definition}{Definition}

\newtheorem{conjecture}{Conjecture}
\newtheorem{proposition}[theorem]{Proposition}

\newcommand{\bg}{{\boldsymbol{g}}}

\newcommand{\bA}{{\boldsymbol{A}}}
\newcommand{\bC}{{\boldsymbol{C}}}
\newcommand{\bS}{{\boldsymbol{S}}}

\newcommand{\bW}{{{\boldsymbol{{W}}}}}
\newcommand{\bw}{{\boldsymbol{w}}}

\newcommand{\bV}{{\boldsymbol{V}}}

\newcommand{\bzero}{\mathbf{0}}
\newcommand{\bone}{\mathbf{1}}

\newcommand{\bM}{{\boldsymbol{m}}}

\newcommand{\bMM}{{{\boldsymbol{M}}}}

\newcommand{\bv}{{\boldsymbol{v}}}

\newcommand{\bZ}{{\boldsymbol{Z}}}

\newcommand{\bomega}{{\boldsymbol{\omega}}}

\newcommand{\bz}{{\boldsymbol{z}}}
\newcommand{\bx}{{\boldsymbol{x}}}

\newcommand{\bX}{{\boldsymbol{X}}}
\newcommand{\bI}{{\boldsymbol{I}}}

\newcommand{\by}{{\boldsymbol{y}}}

\newcommand{\bxi}{{\boldsymbol{\xi}}}

\newcommand{\bu}{{\boldsymbol{u}}}

\renewcommand{\vec}[1]{\boldsymbol{#1}}

  \providecommand{\R}{\mathbb{R}} % Reals
  \makeatletter
  \def\sign{\@ifnextchar*{\@sgnargscaled}{\@ifnextchar[{\sgnargscaleas}{\@ifnextchar{\bgroup}{\@sgnarg}{\sgn} }}}
  \def\@sgnarg#1{\sgn\rbr{#1}}
  \def\@sgnargscaled#1{\sgn\rbr*{#1}}
  \def\@sgnargscaleas[#1]#2{\sgn\rbr[#1]{#2}}
  \makeatother

  \providecommand{\bb}{\bm{b}}
  
  \providecommand{\dd}{\bm{d}}

\providecommand{\E}{\mathbb{E}}%
  \providecommand{\cH}{\mathcal{H}}

  \providecommand{\cO}{\mathcal{O}}

\newcommand{\defeq}{\mathrel{\mathop:}=}


\newcommand{\speedup}[1]{{\color{gray}(\ifdim #1 pt > 0.3pt #1\else $< #1$\fi{}$\times$)}}
\usepackage{amsfonts}
\usepackage{amsmath}
\usepackage{amsthm}
\usepackage{amssymb}
\usepackage{dsfont}

\usepackage{xcolor}
\usepackage{color}
\usepackage{graphicx}

\usepackage{verbatim}
\usepackage{xspace} %

\usepackage{enumerate}
\usepackage{enumitem}

\makeatletter
\newsavebox{\@brx}
\newcommand{\llangle}[1][]{\savebox{\@brx}{\(\m@th{#1\langle}\)}%
  \mathopen{\copy\@brx\mkern2mu\kern-0.9\wd\@brx\usebox{\@brx}}}
\newcommand{\rrangle}[1][]{\savebox{\@brx}{\(\m@th{#1\rangle}\)}%
  \mathclose{\copy\@brx\mkern2mu\kern-0.9\wd\@brx\usebox{\@brx}}}
\makeatother

\providecommand{\abs}[1]{\left\lvert#1\right\rvert}
\providecommand{\norm}[1]{\left\lVert#1\right\rVert}

  \providecommand{\R}{\mathbb{R}} %

  \providecommand{\Eb}[1]{{\mathbb E}\left[#1\right] }

  \providecommand{\bb}{\mathbf{b}}
  
  \providecommand{\dd}{\mathbf{d}}

  \providecommand{\cH}{\mathcal{H}}

  \providecommand{\cO}{\mathcal{O}}

\renewcommand{\vec}{\mathbf}
  \usepackage{bm}

\providecommand{\mycomment}[3]{\todo[caption={},size=footnotesize,color=#1!20]{\textbf{#2: }#3}}%
\providecommand{\inlinecomment}[3]{%
  {\color{#1}#2: #3}}%
\newcommand\commenter[2]%
{%
  \expandafter\newcommand\csname i#1\endcsname[1]{\inlinecomment{#2}{#1}{##1}}
  \expandafter\newcommand\csname #1\endcsname[1]{\mycomment{#2}{#1}{##1}}
}

  \definecolor{mydarkblue}{rgb}{0,0.08,0.45}
  \usepackage{hyperref}
  \hypersetup{ %
    pdftitle={},
    pdfauthor={},
    pdfsubject={},
    pdfkeywords={},
    pdfborder=0 0 0,
    pdfpagemode=UseNone,
    colorlinks=true,
    linkcolor=mydarkblue,
    citecolor=mydarkblue,
    filecolor=mydarkblue,
    urlcolor=mydarkblue,
    pdfview=FitH}

\usepackage{url}

\author[1]{Emanuele Troiani}
\author[1,2]{Yatin Dandi}
\author[3]{Leonardo Defilippis}
\author[1]{\\ Lenka Zdeborová}
\author[3]{Bruno Loureiro}
\author[2]{Florent Krzakala}

\affil[1]{Statistical Physics Of Computation Laboratory, \'Ecole Polytechnique F\'ed\'erale de Lausanne (EPFL)}
\affil[2]{Information Learning and Physics Laboratory, \'Ecole Polytechnique F\'ed\'erale de Lausanne (EPFL)}
\affil[3]{Departement d'Informatique, \'Ecole Normale Sup\'erieure, PSL \& CNRS}

\title{Fundamental computational limits of weak learnability in high-dimensional multi-index models}

\begin{document}

\maketitle

\begin{abstract}
  Multi-index models --- functions which only depend on the covariates through a non-linear transformation of their projection on a subspace --- are a useful benchmark for investigating feature learning with neural nets. This paper examines the theoretical boundaries of efficient learnability in this hypothesis class, focusing  on the minimum sample complexity required for weakly recovering their low-dimensional structure with first-order iterative algorithms, in the high-dimensional regime where the number of samples  $n\!=\!\alpha d$ is proportional to the covariate dimension $d$. Our findings unfold in three parts: (i) we identify under which conditions a \textit{trivial subspace} can be learned with a single step of a first-order algorithm for any $\alpha\!\!>\!\!0$; (ii) if the trivial subspace is empty, we provide necessary and sufficient conditions for the existence of an {\it easy subspace} where directions that can be learned only above a certain sample complexity $\alpha\!>\!\alpha_c$, where
  $\alpha_{c}$ marks a computational phase transition.
  In a limited but interesting set of really hard directions --akin to the parity problem-- $\alpha_c$ is found to diverge. Finally, (iii) we show that interactions between different directions can result in an intricate hierarchical learning phenomenon, where directions can be learned sequentially when coupled to easier ones. We discuss in detail the {\it grand staircase} picture associated to these functions (and contrast it with the original staircase one). Our theory builds on the optimality of approximate message-passing  among first-order iterative methods, delineating the fundamental learnability limit across a broad spectrum of algorithms, including neural networks trained with gradient descent, which we discuss in this context.
\end{abstract}

\section{Introduction}
\label{sec:setting}

A fundamental property of neural networks is their ability to learn features and adapt to relevant structures in high-dimensional noisy data. However, our mathematical understanding of this mechanism remains limited. A popular model for studying this question is \emph{multi-index models}. Multi-index functions are a  class of statistical models encoding the inductive bias that the relevant directions for prediction depend only on a low-dimensional subspace of the covariates ${\bf x} \in\mathbb{R}^{d}$:
\begin{equation} \label{first-eq}
    y = g(\bW\bx), \qquad \bW\in\mathbb{R}^{p\times d},
\end{equation}
where the mapping $g: \R^p \rightarrow \R$ can further be stochastic i.e. depend on independent noise $\xi \in \R^{p'}$ for some finte $p'$.
They define a rich class of hypotheses  \citep{Li1991}, containing many widely studied functions in the statistical learning and theoretical computer science literature. The simplest instance is the  {\it Linear model  ($p\!=\!1$)} given by a linear link function $g(z) \!=\! z$, or its noisy version $g(z, \xi) \!=\! z \!+\! \sqrt{\Delta} \xi$, with $\xi\!\sim\! {\cal N}(0,1)$. In the {\it Single-index model ($p=1$)}, also known as {\it generalised linear model}\footnote{More precisely,  \emph{single-index model} is employed in a context where both the direction $\bw$ and link function $g$ are learned, while \emph{generalised linear model} is used when the link function is known and only the weights are learned. In this work, we use both interchangeability.}, $g(z)$ can be an arbitrary, possibility stochastic, link function, including some widely studied problems such as {\it phase retrieval} $g(z)\!=\!z^{2}$ and the perceptron $g(z)={\rm sign}(z)$.
Some example of multi-index models include the {\it Two-layer neural network ($p>1$)}, given by a link function $g(z_{1},\dots, z_{p}) = \sum_{k=1}^{p}a_{k}\sigma(z_{k})$; {\it Polynomial functions ($p>1$)}, given by any linear combination of products of $z_{1},\dots,z_{p}$, 
for example $g(z_{1},z_{2},z_{3}) = 1+z_{1}^{2}z_{2}^{2}-2z_{1}+z_{2}z_{3}^{3}$; and  {\it s-sparse parities ($p>1$)} that can be embedded in a multi-index model by taking a link function $g(z_{1},\dots,z_{p}) = \prod_{k\in I}{\rm sign}(z_{k})$ for any subset $I\subset[p]$ of size $|I|=s\leq p$.
Training a multi-index model typically translates to a non-convex optimization problem. Several authors have thus used multi-index models as a test-bed for understanding the behavior of neural nets and gradient-descent in non-convex, high-dimensional contexts, e.g. \citep{saad1995line,saad_1996, Arous2021, abbe22a, abbe23a, Veiga2022, ba2022high, arnaboldi23a, collinswoodfin2023hitting, Damian2023, bietti2023learning, moniri2023theory, berthier2024learning}.

 To serve as a useful benchmark, it is necessary to have an idea of the fundamental limit of learnability in such models. This translates to the question of how many observations from the model in \cref{first-eq} are required to obtain a better-than-random prediction within a class of algorithms, also known as \emph{weak learnability}. This can be studied both \emph{statistically} (within the class of all, including exponential time algorithms) or \emph{computationally} (restricted to a particular computational class, such as first-order algorithms). In the single-index case ($p\!=\!1$), weak learnability has been heavily studied under probabilistic assumptions for the weights and data distribution (e.g. i.i.d. Gaussian or uniformly in the sphere). The statistical threshold for learnability has been characterised by 
\cite{Barbier2019} 
when the covariate dimension $d$ is large. Optimal computational thresholds for the class of first-order iterative algorithms were also derived in \citep{mondelli2018fundamental, luo2019optimal, celentano20a, Maillard2020}. Algorithmic results were also proven for other computational models: \cite{damian2022neural} provided a lower bound $n \geq d^{\max(1,\sfrac{\ell}{2})}$ under the Correlational Statistical Query (CSQ) model, comprising algorithms that take queries of the type 
$\mathbb{E}[y\varphi(\bx)]$. Here $\ell$ denotes the {\it information exponent} \citep{Arous2021} of the target defined as the first non-zero integer $\ell \!>\! 0$ in the Hermite expansion of $g$: $   \ell \!=\!  \min \{j\!\in\!\mathbb{N}: \langle g,\cH_j\rangle_{\gamma}\neq 0 \}.
$

A related notion is that of staircase functions \citep{abbe22a,abbe23a}, which characterizes the sample complexity for sequential learning of directions under the under the CSQ model or online SGD. For staircase functions, the CSQ sample complexity is governed by the ``leap", which informally equals the maximum jump in degree (in the Hermite expansion) conditioned on previously learned directions. For instance the target $g(z_1z_2z_3) = z_1+z_1z_2+z_1z_2z_3$ has leap $1$, since the term $z_1z_2$ is linear conditioned on $z_1$ while $z_1z_2z_3$ is linear conditioned on $z_1,z_2$.  Additionally \cite{damian2024computationalstatistical} provided results for the Statistical Query (SQ) model, which allows for more flexible queries of the type $\mathbb{E}[\varphi(\bx,y)]$, and a lower sample complexity $n\geq d^{\max(1,\sfrac{\kappa}{2})}$ with $\kappa\leq \ell$ defining the \emph{generative exponent}. As we shall see, our results allows to reconcile this picture with an {\it extended} version of staircase functions that we refer to as {\it grand staircase functions}.

Aside from the particular case of committee machines \citep{Aubin2018, pmlr-v125-diakonikolas20d, goel2020superpolynomial, chen2022hardness}, results for general multi-index models $p>1$ are scarce, as they crucially depend on the way different directions are coupled by the link function. The goal of the present work is precisely to close this gap for the class of Gaussian multi-index models in the proportional, high-dimensional regime, and to provide a classification of how hard it is to learn feature directions from data in multi-index models. More precisely, our {\bf main contributions} are:
\begin{itemize}[leftmargin=*, noitemsep,wide=1pt]
    \item We analyse the computational limits of weak learnability (Def.~\ref{def:weak}) for Gaussian multi-index models in the class of first-order methods, in the high-dimensional limit when $d,n\!\to\!\infty$ at fixed ratio $\alpha\!=\!\sfrac{n}{d}$ and constant number of indices $p\!=\!\Theta(1)$.
    Our analysis leverages the optimality of Bayesian approximate message passing algorithms (AMP) \citep{Donoho2009,rangan2011generalized} among first-order iterative methods \citep{zdeborova2016statistical,celentano20a,montanari2022statistically}, allowing us to delineate the fundamental learnability limits across a wide range of algorithms, including neural nets trained with gradient descent. 
    
    \item We provide a classification of which directions are computationally \emph{trivial}, \emph{easy} or \emph{hard} to learn with AMP from linearly many samples: \textit{trivial directions} can be learned with a single AMP iteration for any $\alpha\!>\!0$. When no such direction exists, we show there might be \emph{easy directions} allowing recovery from arbitrarily small side-information $\sqrt{\lambda} \!>\! 0$ in $\mathcal{O}(\log \sfrac{1}{\lambda})$ steps with at least $n\!=\! \alpha_c d$ samples, where $\alpha_c$ marks a computational transition. Heuristically, we expect recovery along easy directions from random initialization with $\alpha \!> \!\alpha_c $ in $\mathcal{O}(\log d)$ iterations and confirm this numerically through extensive simulations. \footnote{Note that $\alpha_c$ may or may not coincide with the statistically optimal threshold. In no way do we imply the existence or absence of a computational-to-statistical gap.} We provide a sharp formula for the critical sample complexity $\alpha_c$ above which easy directions are weakly learnable. It is conjectured that no efficient iterative algorithm can succeed for $\alpha\!<\!\alpha_c$. The \emph{hardest} multi-index problems are those where $\alpha_c$ diverges, and where none of the directions can be learned from $n\!=\! \mathcal{O}(d)$ samples.
    
    \item We discuss how interactions among directions can lead to hierarchical learning phenomena, with the hierarchy referring to learning a sequence of growing subspaces. Interestingly, even hard directions can be learned when coupled to easier ones through a hierarchical process, which is reminiscent of the staircase phenomenon for one-pass SGD \citep{abbe23a}. However, the class of functions that are weakly learnable hierarchically with AMP differs from SGD and is, in particular, considerably larger. We characterise this class of \emph{grand staircase functions} and argue it is the relevant computational complexity class for learning multi-index models.  
\end{itemize}
Finally, the code to reproduce our plots, to run the AMP algorithm, and to deploy state evolution  is available on GitHub \url{https://github.com/SPOC-group/FundamentalLimitsMultiIndex}

\paragraph{Further related work --- } With its origins on the classical projection pursuit method \citep{Friedman1974, Friedman1981}, there is an extensive literature dedicated to designing and analysing efficient algorithms to train multi-index models models, such as Isotronic Regression for single-index \citep{Brillinger1982, KalaiS09, Kakade2011} and Sliced Inverse Regression in the multi-index case \citep{dalalyan08a, Yuan2011, Fornasier2012, Babichev2018}.

The case $p=1$ has seen a lot of interest recently. In terms of gradient descent, \cite{Arous2021} has shown that if the link function $g$ is known, one-pass SGD achieves weak recovery in $n=\Theta(d^{\ell-1})$ steps, where $\ell$, known as the \emph{information exponent}, is the first non-zero Hermite coefficient of the link function $g$. In the non-parametric setting ---where $g$ is unknown--- \cite{berthier2024learning} has shown that a large-width two-layer neural network trained under one-pass SGD can learn an "easy" single-index target ($\ell=1$) in $n=\Theta(d)$ steps. This is to be contrasted with full-batch GD, which can achieve weak recovery in $\Theta(1)$ steps with sample complexity $n=\Theta(d)$ even for particular problems with $\ell>1$ \citep{dandi2024benefits} (which is, indeed, the statistical optimal rate \citep{Barbier2019,mondelli2018fundamental}). Similar results have also been proven for abstract computational models. \cite{damian2022neural} has proven a lower bound $n \geq d^{\max(1,\sfrac{\ell}{2})}$ under the Correlational Statistical Query (CSQ) model, comprising algorithms that take queries of the type $\mathbb{E}[y\varphi(\bx)]$. More recently, \cite{damian2024computationalstatistical} has proven a similar result for the Statistical Query (SQ) model, which allows for more flexible queries of the type $\mathbb{E}[\varphi(\bx,y)]$ and hence a lower sample complexity $n\geq d^{\max(1,\sfrac{\kappa}{2})}$ with $\kappa\leq \ell$ defining the \emph{generative exponent}. This turns out to be equivalent to the optimal computational weak recovery threshold of \citep{Barbier2019,mondelli2018fundamental,luo2019optimal,Maillard2020}.

While the situation is less understood, the multicase has also whitenessed a surge of recent interest. In particular, it has been used to understand the behavior of gradiend descent algorithm in neural networks. For instance, \cite{abbe22a, abbe23a} showed that a certain class of \emph{staircase} functions can be learned by large-width two-layer networks trained under one-pass SGD with sample complexity $n=\Theta(d)$. This is in stark contrast to (embedded) $s$-sparse parities, which require $n\geq d^{s-1}$ samples \citep{Blum2003}. \cite{abbe22a, abbe23a} introduced the \emph{leap exponent}, a direction-wise generalisation of the information exponent for multi-index models, and studied the class of \emph{staircase functions} which can be efficiently learned with one-pass SGD. \cite{bietti2023learning} showed that under a particular gradient flow scheme preserving weight orthogonality (a.k.a. \emph{Stiefel gradient flow}), training follow a saddle-to-saddle dynamics, with the characteristic time required to escape a saddle given by the leap exponent. Interestingly, it was then shown that the limit discovered in these set of works could be bypassed by slightly different algorithm, for instance by smoothing the landscapes \cite{Damian2023} or reusing batches multi-time \cite{dandi2024benefits}. The latter paper, in particular, showing that gradient descent could learn efficiently a larger class of multi-index models that previously believed to be possible.  These findings highlight the need of a strict understanding of the limit learnability of these models.

Our approach to establish the limit of computational learnability is based on the study of the approximate message passing algorithm (AMP).  Originating from the cavity method in physics \cite{mezard1989space,kabashima2008inference}, AMP \cite{Donoho2009}, and its generalized version GAMP \cite{rangan2011generalized} are powerful iterative algorithm to study these high-dimensional setting. These algorithm are widely believed to be optimal between all polynomial algorithms for such high-dimensional problems \cite{zdeborova2016statistical,deshpande2015finding,bandeira2018notes,Aubin2018,bandeira2022franz}. In fact, they are provably optimal among all iterative first-order algorithms \cite{celentano20a,montanari2022statistically}, a very large class of methods that include gradient descent.

While these algorithm were studied in great detail for the single index models, see e.g. \cite{Barbier2019,aubin2020exact,aubin2020generalization}, and are at the roots of the spectral method that underline the learnability phase transition in this case \cite{mondelli2018fundamental,Maillard2020}, much less is known in the multi-index case, with the exception of \cite{Aubin2018}, who showed that particular instances of multi-index models known as \emph{committee machines} can be learned with $n=\mathcal{O}(d)$ samples.

Different from our approach  are \citep{pmlr-v125-diakonikolas20d,goel2020superpolynomial,chen2022hardness}, who focused also on peculiar form of the multi-index models (using combination of Relu) and derived worst case bound (in terms of the hardest possible function). Instead, we focus on the typical-case learnability of a given, explicit, multi-index model.

\section{Settings and definitions}
Using standard notations\footnote{We denote by $\mathcal{S}^{+}_{p}$, the cone of positive-semi-definite matrices in $\R^p$ and by $\succ$ the associated partial ordering. For $\bMM\in \mathcal{S}^{+}_{p}$, we denote by $\sqrt{M}$ the matrix-square root. $||\cdot||$ denotes the operator norm and $||\cdot||_{F}$ the Frobenius norm. For two sequences of scalars $f(d),g(d)$, we use the standard asymptotic notation $f(d) = \Theta(g(d))$ to denote that $c \abs{g(d)} \leq \abs{f(d)} \leq C \abs{g(d)}$ for some constants $0 < c \leq C$ and large enough $d$. We denote by  $f(d) = \mathcal{O}(g(d))$ and $f(d) = \Omega(g(d))$ the one-sided asymptotic bounds $\abs{f(d)} \leq C \abs{g(d)}$ and $c \abs{g(d)} \leq \abs{f(d)}$ respectively. }, our main focus in this work will be to study subspace identifiability in the class of \emph{Gaussian multi-index models}.
\begin{definition}[Gaussian multi-index models] 
\label{def:model}
Given a covariate $\bx\sim\mathcal{N}(0,\sfrac{1}{d}\bI_{d})$, we define the class of Gaussian multi-index models as output mappings of the type:
\begin{equation}
\label{eq:def:likelihood}
    y = g(\bW^{\star}\bx)
\end{equation}
where $g:\mathbb{R}^{p}\to\mathbb{R}$ denotes the link function and $\bW^{\star}\in\mathbb{R}^{p
\times d}$ is a weight matrix with i.i.d. rows $\bw^{\star}_{k}\sim\mathcal{N}(0,\bI_{d})$. Note we allow the map $g:\mathbb{R}^{p}\to\mathbb{R}$ to be stochastic i.e $g(\cdot)= \tilde{g}(\cdot, \bxi)$ for random variable $\bxi \in \R^{p}$ independent of $\vec{x}$ having a fixed dimension-independent law such that $\Ea{\norm{\bm \xi}_{2}^2} < \infty$.\footnote{Technically, our results hold for $\bxi \in \R^{p'}$ with $p'=\mathcal{O}(1)$, for example the case of a degenerated noise. But for simplicity of exposition we focus on the $p'=p$ case.}
\end{definition}

Given $n$ i.i.d. samples $(\bx_{i},y_{i})_{i\in[n]}$ drawn as per Definition \ref{def:model}, we are interested in investigating the computational bottlenecks of estimating $\bW^{\star}$ from the samples $(\bx_{i},y_{i})_{i\in[n]}$. Note that reconstructing $\bW^{\star}$ or a permutation of its rows is equivalent from the perspective of the likelihood \cref{eq:def:likelihood}. Therefore, in this work, we will be interested in \emph{weak subspace learnability}, which corresponds to obtaining an estimation of the subspace spanned by $\bW^{\star}$ better than a random estimator. This can be defined in an invariant way: 
\begin{definition}[Weak subspace recovery] 
\label{def:weak}
 Let $V^{\star} \subset \R^p$ denote a subspace spanned by vectors representing components along $\bW^{\star}$ such that each $\vec{v} \in V^{\star}$ maps to a vector $\vec{v}_d$ in $\R^d$ through the map $\vec{v}_d=(\bW^{\star})^\top \vec{v}$. Given an estimator $\hat{\bW}\in\mathbb{R}^{p\times d}$ of $\bW^{\star}$ with $\norm{\hat{\bW}_i}^2_F=\Theta(d)$,  we have weak recovery of a $V^{\star} $ if: 
\begin{align}
   \underset{\bv \in \R^p, \norm{\bv}=1}{\inf} \norm{\frac{\hat{\bW} (\bW^{\star})^\top\bv}{d}}= \Theta(1). 
\end{align}
with high probability as $d\to\infty$. 
\end{definition}

Our main tool for characterising the computational bottlenecks in the  multi-index problem is an {\bf \emph{Approximate Message Passing}} (AMP, \Cref{alg:gamp}) acting on ${\bf B} \in \mathbb R^{d \times p}$ and  ${\bf \Omega} \in \mathbb R^{n \times p}$ tailored to our needs:
\begin{align}
    {\bm \Omega }^{t} &= \bX f_t({\bf B}^t) - g_{t-1}({\bm \Omega}^{t-1},{\bf y})\bV_t\\
    {\bf B}^{t+1} &= \bX^T g_t({\bm \Omega}^t,{\bf y}) + f_{t} ({\bf B}^{t})\bA_t       
\end{align}
where the denoisers $g_t(.,y_i)$ and $f_t(.)$ are vector-valued mapping applied row-wise on ${\bf \omega}_i \in \mathbb R^p = {\bf \Omega}_{.i}$ and ${\bf b}_i \in \mathbb R^p = {\bf B}_{.i}$. Here $\hat \bW^t\!\in\!\mathbb R^{p\times d}= f_t({\bf B}^t)^\top$ are the estimator of the weights matrix $\bW^\star$,  $g_t({\bm \Omega}^t) \in \mathbb R^{n \times p}$ are an estimate of the pre-activations, while  $\bA_t$ and $\bV_t$ are the so-called Onsager terms (see Sec.\ref{sec:amp_opt}).

AMP only involves matrix multiplication by $\bX\in\mathbb{R}^{n\times d}$ (and its transpose) 
and thus belongs to the class of first-order algorithms with linear running time of $\mathcal{O}(nd)$.  The key property of AMP is that for well-chosen $g_t$ and $f_t$ (using Bayesian denoisers) it is provably optimal within the class of first-order methods \citep{celentano20a}. In our case, this require $g_t (\bomega,y) = \bg_{\rm out}(y,\bomega,\bV)\in\mathbb{R}^p$ as in eq. (\ref{eq:out_denoiser}), while $f_t$ is just a Gaussian Bayesian denoiser  (see  Sec.\ref{sec:amp_opt}). Learnability for the optimal AMP thus implies a computational lower bound in the class of first-order methods, including in particular popular machine learning algorithms such as gradient descent (SG). For the Gaussian multi-index estimation problem \ref{def:model}, $\bg_{\rm out}$ is  given by the optimal denoiser of an effective $p$-dimensional problem $Y = g(\bV^{\sfrac{1}{2}}\bZ+\bomega)$ with $\bZ\sim\mathcal{N}(\bzero, \bI_{p})$, which reads
\begin{equation}\label{eq:out_denoiser}
\begin{split}
    \bg_{\rm out}(y,\bomega,\bV) = \mathbb{E}[ \bZ |Y=y] = 
    \frac{\int_{\mathbb{R}^{p}} \dd \bz \,e^{-\frac{1}{2}(\bz-\bomega)^{\top}\bV^{-1}(\bz-\bomega)}P(y|\bz)V^{-1}(\bz-\bomega)}{\int_{\mathbb{R}^{p}} \dd \bz \, e^{-\frac{1}{2}(\bz-\bomega)^{\top}\bV^{-1}(\bz-\bomega)}P(y|\bz)}, 
\end{split}
\end{equation}
where the conditional expectation is defined through the output channel $Y= g(\bV^{\sfrac{1}{2}}\bZ+\bomega)$. For stochastic mappings $g(\cdot) = g(\cdot, \xi)$, the expectation is w.r.t both $Z, \xi$ conditioned on $Y$. 
Our results rely on the following assumption on the above denoiser, translating to weak assumptions on $g, \xi$: 
\begin{assumption}\label{ass:link}
(a) $\bg_{\rm out}\!:\!\R^{p^2+p+1}\rightarrow \R^p \!\in\! \mathcal{C}^2$.
(b) the mapping $\hat{g}_{\rm out}: \R^{p+{p'}} \rightarrow \R^{p}$
defined by ${\hat{g}}_{\rm out}(\mathbf{h}, {\bm \xi}) = \bg_{\rm out}(\tilde{g}((\mathbf{h}, {\bm \xi}),\bomega,\bV)$ is pseudo-Lipschitz of finite-order for all $\bomega,\bV \in \R^p, \R^{p\times p}$. Here, ${\bm \xi} \!\in\! \R^{p'}$ denotes independent noise satisfying $\Ea{\norm{\bm\xi}^2} \! <\! \infty $ as per Def. \ref{def:model}.
\end{assumption}

The above denoiser can also be derived as an exact high-dimensional approximation of the Belief Propagation (BP) algorithm for the estimation of the marginals of the Bayesian posterior distribution:
\begin{equation}
    \label{eq:def:posterior}
    p(\bW|\bX,\by) \propto \prod\limits_{i=1}^{n}\delta(y_{i}-g(\bW\bx_{i}))\prod_{k=1}^{p}\mathcal{N}(\bw_{k}|\bzero,\bI_{d}).
\end{equation}
The Bayes-AMP algorithm was derived for Gaussian multi-index models by \cite{Aubin2018}, although it has been analysed only for the particular link function $g(\bz)= {\rm sign}(\sum_{k\in[p]}{\rm sign}(z_{k}))$, known as the \emph{committee machine}. Our goal in this work is to leverage \Cref{alg:gamp} in order to provide a sharp classification of which link functions $g$ are computationally challenging to learn. Our analysis is based on two remarkable properties that make AMP a particularly useful tool for studying high-dimensional estimation. The first property is that for any $t=\Theta(1)$, in the high-dimensional limit $d\to\infty$ the performance of AMP can be tracked without actually running the algorithm.  This result, known as \emph{state evolution}, makes AMP mathematically tractable in high-dimensions \citep{Bayati2011}. 
The second property is its optimality with respect to Bayesian estimation (see SI, \ref{sec:amp_opt}). For multi-index models, the state evolution equations were derived by \cite{aubin2020exact} and rigorously proven by \cite{Gerbelot}. It provides, in particular, an exact characterization of the asymptotic overlaps and prediction error:
\begin{lemma}[State evolution \citep{Aubin2018,Gerbelot}] 
\label{thm:se}
Let $(\bx_{i},y_{i})_{i\in[n]}$ denote $n$ {\it i.i.d.} samples from the multi-index model eq.(\ref{def:model}). Run  AMP from random initialization $\hat{\bW}^{0}\!\in\!\mathbb{R}^{p\times d}$ with $\hat{\bw}_{k}^{0}\overset{\text{i.i.d}}\sim\mathcal{N}(\bzero,\bI_{d})$. Denote by $\hat{\bW}^{t}$ the resulting estimator at time $0\leq t \!\leq\! T$. Then, in the high-dimensional limit $n,d\to\infty$ with fixed ratio $\alpha\!=\!\sfrac{n}{d}$, constant $p$ \& any finite time $T$, the limiting overlaps satisfy:
\begin{equation}
     \sfrac{1}{d}\hat{\bW}^{t}{\hat{\bW}^{t\top}} \xrightarrow{P} \bMM^{t}, \, \sfrac{1}{d}\hat{\bW}^{t}{\bW^{\star\top}} \xrightarrow{P} \bMM^{t}, 
\end{equation}
with $\bMM^{t}$ satisfying the \emph{state evolution equations} from initial condition $\bMM^{0}$ iterated with $\bMM^{t+1} \!=\! F(\bMM^{t})$:
 \begin{equation} 
 \label{eq:replica_equation}
    F(\bMM^{t}) = {G}\! \left(\!\alpha\,\mathbb{E}\Big[\bg_{\rm out}\!\left(Y^{t},\sqrt{\bMM^{t}}\bxi,\bI_{p}-\bMM^{t}\right)^{\otimes 2}
    \Big]
    \right).
\end{equation}
where $G(\bX) = (\bI_{p} + \bX)^{-1}\bX$ and the expectation is taken over the following effective process 
\begin{equation}
\label{eq:effective}
Y^{t} = g\left((\bI_{p}-\bMM^{t})^{\sfrac{1}{2}}\bZ+{\bMM^{t}}^{\sfrac{1}{2}}\bxi\right),
\end{equation}
with $\bZ, \bxi\sim\mathcal{N}(0,\bI_{p})$ independently from $\bZ$. The asymptotic mean-squared error on the label prediction is then given by:
\begin{equation} 
\nonumber
 \mathbb{E}_{\boldsymbol{x},y} \left[\left(y-g\left(\hat{\boldsymbol{W}}^{t}(\boldsymbol{X},\boldsymbol{y})\boldsymbol{x}\right)\right)^2\right]  \xrightarrow{P} \E[(Y^{t}-g(\bZ))^2], 
\end{equation}
where the expectation is taken over the effective estimation process eq.(\ref{eq:effective}) and $\xrightarrow{P}$ denotes convergence in probability w.r.t 
$\boldsymbol{X},\boldsymbol{y},\hat{\bW}^{0}$ as $n,d \rightarrow \infty$.
\end{lemma}

\section{The trivial subspace (\texorpdfstring{$\alpha_c=0$}{})}
\label{sec:trivial}

The state evolution \cref{eq:replica_equation} maps the problem of characterising the computational bottlenecks of first-order methods for high-dimensional Gaussian multi-index models to the study of the deterministic, $p$-dimensional dynamical system $\bMM^{t+1}=F(\bMM^{t})$. A starting point is identifying its fixed points and their basins of attraction.  In the absence of any prior information on $\bW^{\star}$ aside from its distribution, one cannot do better than taking $\hat{\bw}_{k}^{t=0}\sim\mathcal{N}(\bzero,\bI_{d})$ with $k\in[p]$ independently at random from the prior. With high-probability, at initialization, the elements of the overlap matrix $\sfrac{1}{d}\hat{\bW}^{0}{\bW^{\star}}^{\top}$ are $\Theta(d^{-\sfrac{1}{2}})$ element-wise. The asymptotic overlap for an uninformed initial condition is thus  $\bMM^{0}=0$, a null-rank matrix. If $\bMM^{0}=0$ is not a fixed point, then $\bMM^{1} \succ 0$, implying the weak recovery of a subspace of dimension $k={\rm rank}(\bMM^{1})>0$ with just a {\it single step of AMP}. 
\begin{lemma}[Existence of uninformed fixed point] 
\label{thm:unin:existence}
$\bMM=\bzero\in\mathbb{R}^{p\times p}$ is a fixed point of \cref{eq:replica_equation} if and only if the following condition holds almost surely over $Y$:
\begin{equation}
\label{eq:trivial}
     \bg_{\rm out}(Y,\bomega=\bzero,\bV=\bI_p) = \mathbb{E}[\bZ|Y] = \bzero,
\end{equation}
\end{lemma}
As long as the conditional expectation above is not zero almost surely, then AMP weakly learns a non-empty subspace immediately in the first iteration for any number of samples $n=\Theta(d)$. For this reason, we refer to this subspace as a \emph{trivial subspace}:
\begin{definition}[Trivial subspace] \label{def:trivial}
We define $H^\star_T \subseteq \R^p$ as the subspace spanned by $\vec{v} \in \R^p$ such that the following holds almost surely over $Y=g(\bZ)$ with $\bZ \sim \mathcal{N}(\vec{0}, \bI_p)$: 
\begin{equation}
    \label{eq:def:existence}
   \bg_{\rm out}(Y, \bzero, \bI_{p})^\top \vec{v} =  \lim_{d \rightarrow \infty}\mathbb{E}[\langle \vec{v}^\top \bW^\star, \vec{x}\rangle|y=Y] = 0,
\end{equation}
where the expectation is w.r.t the joint measure $p(\vec{x},y)$ defined in Def. \ref{def:model} eq.(\ref{eq:effective}). The trivial subspace $T^\star$ is the orthogonal complement of $H^\star_T$.
\end{definition}

\begin{theorem}\label{thm:triv}
For any $\alpha>0$, with high-probability as $d\rightarrow \infty$, the AMP algorithm with the Bayes-optimal choice of $f_t,g_t$ recovers $T^\star$ as per \cref{def:weak} in a single iteration.
\end{theorem}
Note that for single-index models ($p=1$), the condition in \cref{eq:trivial} reduces to the one derived by \citep{mondelli2018fundamental, Barbier2019, Maillard2020}. Interestingly, this is {\it exactly} the same condition appearing in \cite{damian2024computationalstatistical} for weak learnability of single-index models in the SQ model; see eq.~(3) therein.
To make Lemma \ref{thm:unin:existence} concrete, let us look at a few examples (A detailed derivation of these examples is presented in  \Cref{app:examples}):
\begin{enumerate}[leftmargin=*,noitemsep,wide=1pt]
    \item For single-index models ($p=1$), $T^{*}$ is one dimensional if and only if $g$ is non-even, e.g. $g(z) = {\rm He}_{3}(z)$. This follows from requiring that $g_{\rm out}(y,0,1)\neq 0$ for at least one value of $y$. In particular, on any open interval where $g_{\rm out}$ is invertible we have $g_{\rm out}=g^{-1}$. 
    \item For a linear multi-index model, $g(\bz) =\sum_{i=1}^{p}z_{i}$, $T^{\star}$ is spanned by $\bone_{p}\in\mathbb{R}^{p}$ (all-one vector).
    \item For a committee  $g(\bz)\!=\!\sum_{i=1}^{p}{\rm sign}(z_{i})$, the trivial subspace $T^{\star}$ is again 1-d, spanned by $\bone\!\in\!\mathbb{R}^{p}$.     
    \item For monomials $g(\bz) = z_{1}\dots z_{p}$, the trivial subspace $T^{\star}$ is non-empty if and only if $p=1$. 
    \item For leap one staircase functions \citep{abbe23a}:
    \begin{equation}        
    \label{eq:staircase}
        g(\bz) = z_{1}+z_{1}z_{2}+z_{1}z_{2}z_{3}+\dots
    \end{equation}
    The trivial subspace is $T^{\star}=\mathbb{R}^{p}$ and is spanned by the canonical basis. In other words, AMP learns all the directions with a {\it single step} for any $\alpha>0$.
\end{enumerate}
Lemma 3.1 can also be related to computational models based on queries, such as SQ learning \citep{Kearns1998}: the denoiser $\bg_{\rm out}$ can indeed be interpreted as a non-linear transformation on the labels $y \!\mapsto\!  \bg_{\rm out}(y, \bzero, \bI_{p})$. From this perspective, the statement on the condition for the existence of a non-empty trivial subspace (\ref{eq:def:existence}) translates to the condition 
$\Ea{\bg_{\rm out}(y, \bzero, \bI_{p})^\top \vec{v} \langle  \vec{v}^\top \bW^\star, \vec{x}\rangle}=\Ea{\mathbb{E}[\langle \vec{v}^\top W^\star, \vec{x}\rangle|Y=y]^2} \neq 0$
where $\bv\in T^{\star}$. The left-hand side can be seen as a statistical query of the type $\mathbb{E}[\varphi(y)\psi(x)]$ with label pre-processing $\varphi = \bg_{\rm out}$. The fact that this linear correlation in the transformed labels is non-vanishing implies that one can weakly recover $\vec{v}$ through a tailored spectral method \citep{mondelli2018fundamental,luo2019optimal}. In fact the denoiser $\bg_{\rm out}$ is the optimal such transformation in the sense that when $\bg_{\rm out}$ fails to obtain a linear correlation along $\vec{v}$, i.e when $\vec{v} \in H^\star$, then no transformation can:
\begin{lemma}\label{lem:opt_t}
 $\vec{v} \in H^\star$ if and only if for any measurable transformation $\mathcal{T}:\R \rightarrow \R$:
 \begin{equation}
     \lim_{d \rightarrow \infty}\Ea{\mathcal{T}(y) \langle \vec{v}^\top \bW^\star, \vec{x}\rangle}=0.
 \end{equation}
 Note that the expectation is with respect to the distribution of the labels $y$ (not the effective problem). 
 Equivalently, $\vec{v} \in T^\star$ if and only if there exists a transformation such that $\lim_{d \rightarrow \infty}\Ea{\mathcal{T}(y) \langle \vec{v}^\top \bW^\star, \vec{x}\rangle} \neq 0$.
\end{lemma}

\paragraph{Remarks on GD ---}  The characterization of $H^\star$ and $T^\star$
allows us to contrast the (optimal) performance of AMP with that of gradient-descent based methods on neural nets: First, note that when $T^\star$ is empty, then $\Ea{\mathcal{T}(y) \langle \vec{v}^\top \bW^\star, \vec{x}\rangle}\!=\!0$ for all $\boldsymbol{v} \in \R^{p}$. This is precisely  the condition for the generative exponent \citep{damian2024computationalstatistical} being strictly larger than $1$.
This implies that one-pass SGD ---even after applying any transformation $\mathcal{T}$--- fails to obtain weak recovery  with $\mathcal{O}(d)$ samples (see Thms. 1.3 and 1.4 in \cite{Arous2021}).
The converse question is interesting: For non-emtpy $T^{\star}$, can GD recover $T^{\star}$ with just $\mathcal{O}(d)$ steps? In the absence of any transformations of the labels, this remains false for online SGD \citep{Arous2021,abbe22a}, unless
the {\it information} exponent of $g$ is $1$. The situation is different when reusing batches: \cite{dandi2024benefits} indeed showed that full batch GD on a two-layer neural network implicitly applies transformation $\mathcal{T}$ to the labels. The learnable subspace in two iterations with $\mathcal{O}(d)$ sample complexity then {\it exactly coincides} with that of $T^{\star}$ in Lemma \ref{lem:opt_t}, up to the restriction to polynomial transformations. A two-layer network with re-use of batches can thus recover the full trivial subspace efficiently
for any $\alpha > 0$.

\section{Phase transitions (\texorpdfstring{$\alpha_c\!>\!0$}{})}
\label{sec:stability}

When the trivial subspace $T^{\star}$ from Def. \ref{def:trivial} is empty, $\bMM=\bzero$ is a fixed point of state-evolution, thus AMP starting from random initialization fails to recover any subspace of $W^\star$ in any {\it finite} number of steps. However, there may exist directions where recovery requires only an infinitesimal additional side-information. To contrast with the trivial subspace, we refer to the directions learnable for arbitrarily small side-information as the \emph{easy directions}.

The learnability of easy subspaces is characterized by AMP's stability, which crucially depends on the sample complexity $\alpha \!=\! n/d$. The stability of $\bMM=\bzero$ is studied by linearising the state evolution \eqref{eq:replica_equation} around $\delta \bMM$ with $\norm{\delta \bMM} \approx 0$
\begin{equation}
F(\bMM) \approx \alpha\mathcal{F}(\delta\bMM) + \mathcal{O}(\norm{\delta\bMM}^2)\,, \nonumber 
\end{equation}
where $\mathcal{F}(\delta\bMM)$ is a linear operator on the cone $\mathcal{S}^+_p$ of PSD matrices of dimension $p$:
\begin{equation} 
\label{eq:stab_theorem}
{
    \!\!\!\! \! \mathcal{F}(\bMM) \defeq \mathbb{E}_{Y}\left[ \partial_{\bomega} \bg_{\rm out}(Y,\bzero,\bI_p)\, \bMM\partial_{\bomega} \bg_{\rm out}(Y,\bzero,\bI_p)^\top  \right].}
\end{equation}
  The above linearization demarcates a threshold on $\alpha$, below which the AMP remains stable:
\begin{lemma}\label{lem:stab_0}
    (Stability of the uninformed fixed point). \label{thm:unin:stability}
    If $\bMM=\bzero\in\mathbb{R}^{p\times p}$ is a fixed point, then it is an {\bf unstable} fixed point of \cref{eq:replica_equation} if and only if $\norm{\mathcal{F}(\bMM)}_F > 0$ and $n>\alpha_c d$, where the critical sample complexity $\alpha_c$ is:
    \begin{equation} 
    \label{eq:alpha_c}
        \frac{1}{\alpha_c} = \sup_{\bMM\in \mathbb{R}^{p\times p},\, \norm{\bMM}_F=1} \norm{\mathcal{F}(\bMM)}_F,
    \end{equation}
    with $\|\cdot\|_F$ denoting the Frobenius norm. Moreover, if $\mathcal{F}(\bMM) \neq 0$, there exists at-least one $\bMM^\star \neq 0 \in \mathcal{S}^+_p$ achieving the above supremum.
While
   if $\mathcal{F}(\bMM) = 0$, then $\bMM=\bzero$ is a stable fixed point for any $n = \Theta(d)$.
\end{lemma}
The above result follows from a generalization of the Perron-Frobenius theorem to the cone $\mathcal{S}^+_p$; see \Cref{appendix:proof_stability} for a proof. Analogously, the operator $\mathcal{F}(\bMM)$ allows us to define the stability threshold along a specific direction $\vec{v} \in\mathbb{R}^{p}$. We define the \emph{easy subspace} as containing the directions where $\bMM=\bzero$ becomes unstable for some $\alpha < \infty$:
\begin{definition}[Easy subspace $E^{\star}$]
\label{def:easy}
    Let $H^*_E$ be the subspace spanned by directions $\vec{v} \in\mathbb{R}^{p}$ such that:
    \begin{equation}
        \vec{v}^\top \partial_{\bomega} \bg_{\rm out}(Y,\bomega=\bzero,\bV=\bI_p)\vec{v} = 0,
    \end{equation}
almost surely over $Y$. We define the easy subspace $E^{\star}$ as the orthogonal complement of $H^{\star}_E$. 
\end{definition}

Modeling side information as additional observations:
\begin{equation}
    {\bf S} = \sqrt{\lambda} {\bf W^\star} +\sqrt{1-\lambda}\bZ, 
    \label{ref:side_info}
\end{equation}
where $\bZ\!\in\!\R^{p \times d}$ has independent entries  $Z_{ij}\!\sim\! \mathcal{N}(0,1)$, we can now  state our main theorem:
\begin{theorem}\label{thm:stab_init}
Let $\bMM_d^{t} \coloneqq \sfrac{1}{d}\hat{\bW}^{t} {\bW^{\star}}^{\top}$ denote the model-target overlap matrix at any finite time $t$.
Suppose that $T^\star={0}$ and consider the AMP algorithm  with the Bayes-optimal choice of $f_t,g_t$. Then, with high probability as $d \rightarrow \infty$:
\begin{itemize}[leftmargin=*, noitemsep,wide=1pt]
    \item [(i)]For $\alpha \geq \alpha_c$, $\exists \delta >0$ such that for any sufficiently small $\lambda$ ,$\bMM_d^{t} \succ \delta \bMM^\star$ for $t = \mathcal{O}(\log \sfrac{1}{\lambda})$, where $\bMM^\star$ is any of the extremizers defined in Lemma \ref{lem:stab_0}. Furthermore, there exists an $\alpha \geq \alpha_c$ and a $\delta >0$ such that $\bMM_d^{t} \succ \delta \bMM_{E^*}$ in $t = \mathcal{O}(\log \sfrac{1}{\lambda})$ iterations, where $\bMM_{E^\star} \in \mathcal{S}^+_p$ spans $E^\star$. 
    \item [(ii)] For $\alpha < \alpha_c$ however, $\bMM^{t}_d=\bzero$ is asymptotically stable i.e. there exist constants $\lambda'<1$ and $C>0$ such that for $\lambda <\lambda'$, $\sup_{t \geq 0} \norm{\bMM^{t}_d} \leq C\sqrt{\lambda}$. 
\end{itemize}
\end{theorem}
Concretely, this implies that for $\alpha < \alpha_c$, not only does AMP fail to find any pertinent directions, but it also fails to improve on the small side-information. For $\alpha > \alpha_c$, however, AMP will develop a growing overlap along a non-empty subspace starting with {\it arbitrarily small} (but finite) side information. Based on the optimality of AMP amongst first-order methods for any finite number of iterations \citep{celentano20a,montanari2022statistically}, $\alpha_c$ thus marks the onset of \emph{computational phase transition} (see S.I. \ref{app:comp_transitions}) for weak learnability in the first-order methods class. 

Note that in Thm. \ref{thm:stab_init} the side information {\rm snr} $\lambda>0$ can be arbitrarily small, but must remain $\cO(1)$. Based on a vast array of evidence from statistical physics literature, we are actually claiming a {\it stronger result}, and believe the following conjecture to be exact:
\begin{conjecture}\label{thm:AMP_random}
AMP initialized randomly will find a finite overlap with the easy directions for $\alpha>\alpha_c$  in ${\cal O}(\log d)$ steps, without side information.
\end{conjecture}
This can be justified by the heuristic $\lambda = \mathcal{O}(\sfrac{1}{\sqrt{d}})$ (as this is the scaling of the correlation with the ground truth at initialization) in Thm. \ref{thm:stab_init}. While is not strictly covered by the state evolution theorem that allow only a finite number of steps, it is however well obeyed in practice in our simulations (see Fig.\ref{fig:dynamics}) and it is a well known observation in many message-passing works that AMP (or belief propagation) find the fixed point after few iterations from a random start (see e.g. \cite{decelle2011inference,zdeborova2016statistical,Barbier2019}).

A rigorous proof, though, remain a difficult open problem that require a non-asymptotic control of the AMP state evolution, see \citep{rush2018finite,li2022non,li2023approximate}
for some recent progress in the context of $\mathbb{Z}_2$ synchronization where the convergence from random start has been established, corroborating the heuristic identification $\lambda = \cO(\sfrac{1}{\sqrt{d}})$ in this case. An alternative possible strategy to provide a recovery algorithm in $\cO(\log(d))$ steps would be to follow \citep{krzakala2013spectral,maillard2022construction} and linearize AMP to reach a power iteration method (with a matrix similar to the one discussed in 
Lemma \ref{lem:opt_t}) so that the resulting operator can be studied with random matrix theory \citep{guionnet2023spectral,mondelli2021approximate}, as was done for the single-index model in \citep{mondelli2018fundamental,lu2020phase,maillard2022construction}. 
We leave these, and the proof of conjecture $1$, open for further studies.

The expression for the weak recovery threshold Eq.(\ref{eq:alpha_c}) in  Lemma \ref{thm:unin:existence} allows to generalise the single-index expression from \citep{mondelli2018fundamental, Barbier2019, Maillard2020}. Indeed, for $p=1$, we have $\partial_{\omega} g_{\rm out}(y,0,1) = \mathbb{E}[{\rm He}_2(Z)|Y=y]$, and therefore $\alpha_c = \mathbb{E}_{Y}\left[\mathbb{E}[{\rm He}_2(Z)|Y]^2\right]^{-1}$ or equivalently
\begin{equation}
\label{eq:weakrev:glm}
    \frac{1}{\alpha_c} = \int_{\mathbb{R}}\dd y\frac{\left[\int_{\mathbb{R}} \frac{\dd z}{2\pi}e^{-\frac{1}{2}z^{2}}(z^{2}-1)P(y|z)\right]^{2}}{\int_{\mathbb{R}} \frac{\dd z}{2\pi}e^{-\frac{1}{2}z^{2}}P(y|z)}
\end{equation}
which is exactly the threshold in \citep{Barbier2019,mondelli2018fundamental,lu2020phase,maillard2022construction,damian2024computationalstatistical}.

We now illustrate \Cref{thm:unin:stability} in a two examples of interest: 

(i)  {\bf The monomial $g(\bz) = z_{1}\dots z_{p}$} with $p>1$ can always be learned with $\alpha > \alpha_{c}(p)$ large enough \citep{chen20a}. For instance, we have $\alpha_{c}(2) \approx 0.5937$, $\alpha_{c}(3)\approx 3.725$ and $\alpha_{c}(4)\approx 4.912$. In \Cref{app:examples} we derive an analytical formula for $\alpha_{c}(p)$ for arbitrary $p$, and show that $\alpha_{c}(p)\sim p^{1.2}$ for large $p$.

(ii) The embedding of the {\bf sparse parity functions}:
\begin{equation}\label{2-p}
   g(\bz) = \prod_{k=1}^{p}{\rm sign}(z_k)\,.
\end{equation}
As this is invariant under permutations of the indices, Lemma \ref{thm:unin:stability} implies the existence of a computational phase transition. In S.I. Sec.\ref{sec:app:parity} we compute analytically the critical value $\alpha_c(p)$: $p\!=\!1$ is equivalent to phase retrieval and $\alpha_c(1) \!=\! \sfrac{1}{2}$. For $p\!=\!2$, we show that $\alpha_c(2) \!=\! \sfrac{\pi^2}{4}$, while $\alpha_c(p)\!=\!+\infty$ for $p\!\geq\!3$. This is illustrated in Fig.\ref{fig:z1_sign} (left) for 2-sparse parity.

\begin{figure*}[t]
    \centering
    \includegraphics{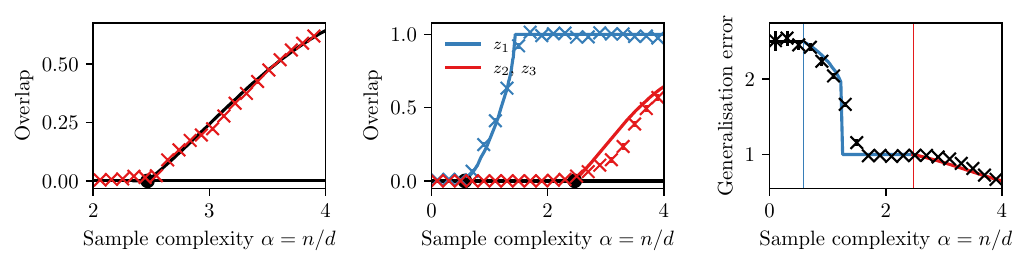}
    \caption{Weak learnability phase transitions for $g(z_{1},z_{2})\!=\!{\rm sign}(z_{1}z_{2})$ (left) and $g(z_{1},z_{2},z_{3})=z_{1}^{2}+{\rm sign}(z_{1}z_{2}z_{3})$ (center and right). Given the permutation symmetry in the models, we display the optimal permutation of the overlap matrix elements reached by AMP.
    ({\bf Left}): Overlaps with the two directions $\sfrac{1}{2}(M_{11}+M_{22})$ as a function of the sample complexity $\alpha=\sfrac{n}{d}$, with the phase transition at $\alpha_c = \pi^2/4$.  The solid black line is the asymptotic theory from state evolution while crosses are averages over $72$ AMP runs with $d\!=\!500$.
    ({\bf Center}): Overlaps with the first direction $|M_{11}|$ ({\color{blue} blue}), and with the second and third one $\sfrac{1}{2}(M_{22}+M_{33})$ ({\color{red} red}) as a function of the sample complexity $\alpha=\sfrac{n}{d}$. Solid lines are the state evolution curves \cref{eq:replica_equation}, and crosses/dots AMP runs  with $d=500$  averaged over $72$ seeds. All other overlaps are zero (black). The two black dots indicate the critical thresholds at $\alpha_1\approx 0.575$ and $\alpha_2=\pi^2/4$. ({\bf Right}) Corresponding generalization error as a function of $\alpha=n/d$. The figure can be reproduced using the code provided (See also \Cref{sec:app:numerics}).}
    \label{fig:z1_sign}
\end{figure*}
The sparse-parity function is a classic example of a computationally hard function in theoretical computer science \citep{Blum1994, Blum2003}.\footnote{Noiseless parity is actually efficiently learnable by Gaussian elimination \citep{Blum1994}. Here we consider computational efficiency in a noise-tolerant sense.} Therefore, it should come as no surprise that $\alpha_{c}(p)=\infty$ for $p\geq 3$, providing a concrete example of what a hard subspace (i.e. that cannot be learned with $n=\Theta(d)$ samples) looks like for AMP. Similar to the discussion of trivial directions in Lemma \ref{lem:opt_t}, we can invoke optimality of AMP to translate our result on the optimal denoiser $y\mapsto \bg_{\rm out}(y,\bzero, \bI_{p})$ to a statement for queries on general label pre-processing transformations:
\begin{lemma}
\label{lemma:He2}
For any $\vec{v} \in H^\star_E$ and any measurable transformation $\mathcal{T}:\R \rightarrow \R$ we have 
\begin{equation}
    \lim_{d \rightarrow \infty }\Ea{\mathcal{T}(y) {\rm He}_2(\langle \vec{v}^\top \bW^\star, \vec{x}\rangle)}=0\,.
\end{equation}

 Note that the expectation is with respect to the distribution of the labels $y$ (not the effective problem from eq.(\ref{eq:effective})).
\end{lemma}

\paragraph{Remarks on GD ---} 
Analogous to Lemma \ref{lem:opt_t},  Lemma \ref{lemma:He2}, implies that the condition $T^\star, E^\star$ being empty (i.e. $\alpha_c=\infty$ is equivalent to the generative exponent \citep{damian2024computationalstatistical} being strictly larger than $2$. The results in \cite{Arous2021,abbe23a} then imply that
even after applying any transformation to the output (as allowed in SQ) online-SGD requires at least $\Omega(d^2)$ samples to achieve weak-rovery, since online-SGD with squared loss requires $\Omega(d^{\kappa-1})$ samples/iterations for weak-recovery of a target function with leap/information exponent $\kappa$ (Theorems 1.3 and 1.4 in \cite{Arous2021}). Furthermore, for single-index models satisfying the condition in Lemma \ref{lemma:He2} (i.e $\alpha_c \!=\! \infty$), the SQ lower-bound predicts a sample complexity requirement of at-least $\Omega(d^{3/2})$ \citep{damian2024computationalstatistical}. Learning such functions with $\mathcal{O}(d)$ samples is thus conjectured to be algorithmically-hard.

It is again a natural question to ask for non-empy $E^\star$, how many samples and iterations do neural networks trained with GD require for achieving weak-recovery. In the absence of label transformations, weak-recovery along $E^\star$ requires $\Omega(d^2)$ unless the information exponent is at-most $2$ \citep{Arous2021,abbe23a}. However, in the presence of transformations, the sample complexity for weak-recovery along $E^\star$ can be reduced to $d \operatorname{polylog} (d)$. \cite{chen2022hardness} already showed that any polynomial can be learned with $d \operatorname{polylog} (d)$ iterations using an SQ algorithm. More-recently, 
\citep{arnaboldi2024repetita,lee2024neural} suggested that $E^\star$ can be recovered in $d \operatorname{polylog} (d)$ steps through modified gradient-based algorithms re-using data. 
It remains an open question, however, to see if $\alpha_d d$ data are enough for an agnostic GD approach.

\section{The grand staircase}
\label{sec:hierarchical}

\begin{figure*}
    \centering
    \includegraphics{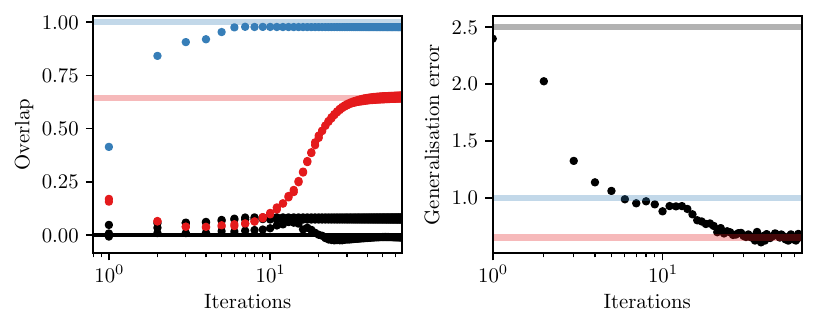}
    \caption{Trajectories of a single finite-size run of AMP with $d = 500$ at $\alpha = 4$ for $g(z_1,z_2,z_3) = z_1^2+{\rm sign}(z_1z_2z_3)$. (\textbf{Left}) Evolution of the overlaps. We display $M_{11}$ in {\color{blue}blue}, $\nicefrac{1}{2}(M_{22}+M_{33})$ in {\color{red}red}, and the off-diagonal overlaps in black. (\textbf{Right}) Evolution of the generalisation error.}
    \label{fig:dynamics}
\end{figure*}

The present classification of \emph{trivial}, \emph{easy} or \emph{hard} directions, based upon the notion of weak learnability, does not say what can be learned {\it after} a given subspace is learned.
We now address this question. Suppose that the estimator $\hat{\bW}^{t}$ has developed an overlap along a subspace belonging to the span of $\bW^{\star}$, resulting in a non-zero overlap $\bMM^t \succ 0$. From that point, the main difference with respect to the previous discussion is that the variable $\bomega$ in the linear operator $\mathcal{F}$ defined in eq.(\ref{eq:stab_theorem}) becomes non-zero (since it is distributed as $\bomega = \sqrt{\bMM}\bxi$). Crucially, this changes the span of $\mathcal{F}(\bMM)$ and hence the stability condition in Lemma \ref{lem:stab_0}. In particular, as we show next, learning some directions might facilitate learning larger subspaces. This  is reminiscent of the specialisation transition in committee machines \citep{saad_1996,aubin2020exact} and of the staircase  for SGD \citep{abbe22a}:
\begin{definition}\label{def:denoise_coup}
Let $U\subset \R^p$. We define $H_{T}^\star(U)$  to be the subspace spanned by $\bv \in U^\perp$ such that 
\begin{equation}
    \bv^\top \bg_{\rm out}(Y,\sqrt{\bMM}_{U}\bxi,\bI-\sqrt{\bMM}_{U}) = 0
\end{equation}
almost surely over $\bxi\sim\mathcal{N}(\bzero, \bI_{p})$ and $Y$ for any $\bMM_U \in \mathcal{S}^+_p$ such that $\operatorname{span}(\bMM_U) = U$.
We define the ``trivially-coupled" subspace $T^\star_U$ for $U$ as the orthogonal complement of $H_{T}^\star(U)$. 

Analogously, let $H^\star_{E}(U)$ be the subspace spanned by directions $\bv \in U^\perp$ such that 
\begin{equation}
    \bv^\top\partial_{\bomega} \bg_{\rm out}(Y,\sqrt{\bMM_U}\bxi,\bI-\sqrt{\bMM}_{U})\bv = 0
\end{equation}
almost  surely over $\bxi$ and $Y$ for any $\bMM_U \in \mathcal{S}^+_p$ such that $\operatorname{span}(\bMM_U)=U$.

When $\bMM_U$ is additionally a fixed point of $\mathcal{F}_{\bMM}$, one can linearise $\mathcal{F}_{\bMM}$ along the orthogonal complement of $U$.
We define the easy-coupled subspace $E^\star_U$ for $U$ as the orthogonal complement of $H^\star_{E}(U)$. Next, suppose that $\bMM_U \in \mathcal{S}^+_p$ with $\operatorname{span}(\bMM_U)=U$ is a fixed-point of $\mathcal{F}_{\bMM}$.
Let $\mathcal{F}_{\bMM_U}$ denote the linearization of $F(\bMM)$ along the orthogonal complement $U^\perp$ at $\bMM=\bMM_U$. We define the grand staircase threshold $\alpha_{\text{gst}}(\bMM_U)$ at $\bMM=\bMM_U$ as $\alpha_{\text{gst}}(\bMM_U) = (\sup_{\bMM^\perp\in U^\perp} \norm{\mathcal{F}_{\bMM_U}(\bMM^\perp)}_F)^{-1}$.
\end{definition}

The definitions above generalise the notions of \emph{trivial} and \emph{easy} subspaces in Def. \ref{def:trivial} \& \ref{def:easy} conditionally on a subspace $U$ that has been previously learned, and characterise the directions whose recovery is enabled upon learning the subspace $U$.
Concretely, upon developing an initial overlap along $U$, the directions in $T^\star_U, E^\star_U$ can be recovered analogous to the recovery of $T^\star, E^\star$ in Theorems \ref{thm:triv}, \ref{thm:stab_init}.
For brevity, we focus the present discussion on the recovery of $E^\star_U$
with vanishing side information, and discuss settings involving the recovery of $T^\star_U$ in Appendix \ref{app:triv_coup}.
\begin{proposition}(informal)\label{prop:hid}
Let $U \subseteq \mathbb{R}^p$ be a subspace such that $E^\star_U$ is non-empty.
Consider AMP iterates with the Bayes-optimal choice of $f_t,g_t$ for sufficiently small $\lambda>0$.  Suppose that $\bMM_d^{t} \coloneqq \sfrac{1}{d}\hat{\bW}^{t} {\bW^{\star}}^{\top}$ is  an approximate fixed point of $F(\bMM)$ in \cref{eq:replica_equation} such that $\bMM_d^{t} \approx \bMM_U$  where $M_U$ spans  $U$. 
Then:
\begin{itemize}[leftmargin=*, noitemsep,wide=1pt]
    \item For $\alpha>\alpha_{\text{gst}}(\bMM_{U})$, AMP recovers $M^*_{M_U}$ in additional $\mathcal{O}(\log \frac{1}{\lambda})$  steps for arbitrarily small $\lambda$, where $M^*_{M_U}$  denotes any matrix in $\mathcal{S}^+_p$ achieving the supremum $\sup_{\bMM^\perp\in U^\perp} \norm{\mathcal{F}_{\bMM_U}(\bMM^\perp)}_F$ as per Definition (\ref{def:denoise_coup}).
    \item For $\alpha<\alpha_{\text{gst}}(\bMM_{U})$ and sufficiently small $\lambda$, AMP remains at the approximate fixed point $\bMM_U$ and fails to gain weak-recovery along $U^\bot$.
\end{itemize}
\end{proposition}
When the easy subspace $E^\star$ defined by Def. \ref{def:easy} and $E^\star_{E^\star_{U}}$ are non-empty, we further show in S.I. \ref{app:hid} that for large enough $\alpha \geq \alpha_c$ AMP recovers $E^\star_{U}$ by first reaching an approximate fixed point spanning $E^\star$ and subsequently ``escaping" along $E^\star_{E^\star_{U}}$ as described above. 
A concrete example of a function displaying this phenomenon is a linear combination between \emph{hard} parity function and an \emph{easy} polynomial:
\begin{equation}
\label{lastexample}
   g(z_1,z_2,z_3) = z_1^2 + {\rm sign}(z_1 z_2 z_3)
\end{equation}
The sign part is a sparse parity with $p=3$, which cannot be learned with $n=\Theta(d)$ samples, but the quadratic polynomial $z_1^2$ component in the function allows weak-recovery of the first component i.e. $U={(1,0,0)}$  as long as $\alpha>1/2$. Hence, conditionally on $U$ the effective multi-index model becomes ${\rm sign}(z_2 z_3)$, which as discussed in Sec.\ref{sec:stability} is an \emph{easy} function. Figure \ref{fig:z1_sign} illustrates this: first, $z_{1}$ is learned at $\alpha_1\approx0.575$. 
Then, for {\it larger} value when $\alpha>\alpha_2$, all directions are learned (see figure \ref{fig:z1_sign}). Knowing $z_1$ makes the \emph{hard} 3-parity an \emph{easy} 2-parity problem.

It is easy to construct a multi-index model where AMP will iterate over any number of such plateaus. For instance, consider the model:
\begin{equation}
    g(\bz) = z_1^2 + {\rm sign}(z_1 z_2 z_3) + {\rm sign}(z_3 z_4 z_5) + \ldots\,.
\end{equation}
After the first plateau to learn $z_1$, there will be one for  $z_2$ and $z_3$, then for $z_4$, $z_5$. Another example (see  Fig.~\ref{fig:committee} in the S.I.) is the neural net target $g(z_{1},z_{2},z_{3})\!=\!{\rm sign}(z_{1})\!+\!{\rm sign}(z_{2})\!+\!{\rm sign}(z_{3})$. At first only the one-dimensional direction spanned by $z_1\!+\!z_2\!+\!z_3$ is learned for any $\alpha>0$. This is simply a linear approximation of the function. Only at $\alpha\!\approx\!4.3$ does the network learns the other directions. Known as the {\it the specialisation transition}, this provide a very early example of such sequential learning \citep{saad_1996,aubin2020generalization}.

We can characterize a sequence of growing subspaces that AMP recovers iteratively, with vanishing side information, for large enough $\alpha$, as follows:

\begin{definition}\label{def:stair}
    Let $E^\star_1 = E^\star \cup T^\star$, and inductively define for $k \in \mathbb{N}$:
    \begin{equation}
E^\star_{k+1}=E^\star_{k} \cup E^\star_{E^\star_{k}} \cup T^\star_{E^\star_{k}}
    \end{equation}
The dimensionality of the subspaces in the above sequence is non-decreasing and therefore, $\exists k$ such that $E^\star_{E^\star_{k}}=\emptyset$. We denote this maximal subspace by $E^\star_{\text{gst}}$ and say that $g$ is grand-staircase learnable if $E^\star_{\text{gst}} = \mathbb{R}^p$.
\end{definition}
Using proposition \ref{prop:hid} and certain simplifying properties of the Bayes-AMP dynamics, we can show (see S.I.) that for large enough $\alpha$, AMP recovers $E^\star_{\text{gst}}$:
\begin{theorem}\label{thm:staircase}
For any $\delta$, there exists $\alpha>\alpha_c$ such that with arbitrarily small side-information $\sqrt{\lambda}$, with high probability as $d \rightarrow \infty$, $\bMM_d^{t} \succ \delta M_{E^\star_{\text{gst}}}$ in $t = \mathcal{O}(\log \sfrac{1}{\lambda})$ iterations, where $M_{E^\star_{\text{gst}}} \in \mathcal{S}^+_p$ spans $E^\star_{\text{gst}}$.
\end{theorem}

\paragraph{Final remarks  ---} 
We stress that grand staircase functions are {\it not} equivalent to the staircase ones of \citep{abbe22a,abbe23a}. Consider for instance
\begin{equation}
    g(\bz) = \rm He_4(z_1) + {\rm sign}(z_1 z_2 z_3)\,.
\end{equation}
As for (\ref{lastexample}), it can be AMP-learned  with $\mathcal{O}(d)$ samples by first recovering the direction corresponding to $z_1$. Unlike  (\ref{lastexample}), however, which is ``leap 2" under the CSQ staircase classification, $g(\bz)$ is ``leap 3". The CSQ lower-bound therefore implies a sample complexity-requirement of $\Omega(d^{1.5})$ with online-SGD learning it in $\mathcal{O}(d^2)$ samples. 

We expect  {\it grand staircase functions} to be  efficiently learnable by two-layers nets if data reusing is allowed (while only {\it standard staircase} ones are learned with single-pass SGD).  \cite{arnaboldi2024repetita} recently provided strong evidence that they can indeed be learned in $\cO(d \log d)$ steps. The  analysis of \cite{joshi2024complexity} is  compatible with this conjecture. We hope our study will spark more work in this direction.

\section*{Acknowledgements}
The authors would like to thank Joan Bruna, Theodor Misiakiewicz, Luca Pesce, and Nati Srebro for stimulating discussions. This work was supported by the Swiss National Science Foundation
under grants SNSF SMArtNet (grant number 212049) and SNSF OperaGOST (grant number 200390) and the Choose France - CNRS AI Rising Talents program.

\bibliographystyle{plainnat}
\bibliography{bibliography}

\newpage 
\appendix
\section*{\huge Supplemental material}

\section{Bayes optimality of AMP} \label{sec:amp_opt}

We present in Alg.\ref{alg:gamp} a detailed version of the AMP algorithm described in the text using the explicit notation of \cite{Aubin2018}. While it was derived by \cite{Aubin2018} as a limit of Belief-Propagation, it can also be directly infered from a multi-dimensional version of the standard AMP/GAMP for linear models \cite{rangan2011generalized,javanmard2013state,zdeborova2016statistical}.  As mention in the main, AMP \cref{alg:gamp} is tailored to sample by estimating the posterior marginals \cref{eq:def:posterior}.  AMP estimates the mean $\omega$ and variance $\mathbb{V}$ of $\mathbf{W}^*\mathbf{x}$ and the mean $\widehat{\mathbf{W}}$ and covariance $\widehat{\mathbf{C}}$ of $\mathbf{W}^*$ through a fixed point iteration. Note that the Onsager terms $\bV_i$ and $\bA_k$ can be made independent of the matrix by noting that the sum concentrate on the average of $\hat C$ and $\partial g$ respectively.

\begin{algorithm}[h]
   \caption{Multi-index AMP}
\begin{algorithmic}
    \label{alg:gamp}
   \STATE {\bfseries Input:} Data $\bX\in\mathbb{R}^{n\times d}$, $\by\in\mathbb{R}^{n}$
   
   \STATE Initialize $\hat{\bW}^{t=0}\in\mathbb{R}^{p\times d}$, $\hat{\bC}_{k}^{t=0} \in \mathcal{S}_{p}^{+}$ for $k\in[d]$, $\bg^{t=0} \in\mathbb{R}^{n\times p}$.
   \FOR{$t\leq T$}
   \STATE \textit{/* Update likelihood mean and variance}
   \STATE $\bV_{i}^{t} = \sum_{k=1}^{d}X^{2}_{ik}\hat{\bC}_{k}\in\mathbb{R}^{p\times p}$; \quad $\vec{\omega}_{i}^{t} = \sum_{k=1}^{d}X_{ik} \hat{\bW}_{k}^{t} - \bV_{i}^{t} \bg_{i}^{t-1} \in\mathbb{R}^{p}$, $i\in[n]$;  
   \STATE $\bg_{i}^{t} = g_{\rm{\rm out}}(y_{i}, \bomega_{i}^{t}, \bV_{i}^{t})\in\mathbb{R}^{p}$ ; \quad $\partial \bg_{i}^{t} = \partial_{\bomega}g_{\rm{\rm out}}(y_{i}, \bomega_{i}^{t}, \bV_{i}^{t})\in\mathbb{R}^{p\times p}$ ; $i\in [n]$
   \STATE \textit{/* Update prior first and second moments}
   \STATE $\bA_{k}^{t} = -\sum_{i=1}^{n}X_{ik}^{2} \partial \bg^{t}_{i}\in\mathbb{R}^{p\times p}$ ; \qquad $\bb_{k}^{t} = \sum_{i\in[n]} X_{ik}\bg_{i}^{t} + \bA_{k}^{t} \hat{\bW}_{k}^{t}$; \quad $k\in[d]$ 
   \STATE $\hat{\bW}_{k}^{t+1} = (\bI_{p} + \bA_{k}^{t})^{-1}\bb_{k}^{t}\in\mathbb{R}^{p}$ ;\qquad $\hat{\bC}_{k}^{t+1} =  (\bI_{p}+\bA_{k}^{t})^{-1}\in\mathbb{R}^{p\times p}$, \quad $k\in[d]$
   
   \ENDFOR
   \STATE {\bfseries Return:} Estimators $\hat{\bW}_{\rm amp}\in\mathbb{R}^{p\times d}, \hat{\bC}_{\rm amp}\in\mathbb{R}^{d\times p\times p}$
\end{algorithmic}
\end{algorithm}

Sampling is, a priori, computationally prohibitive in the high-dimensional limit of interest here. The first impressive feat of AMP, a simple iterative first-order algorithm, is thus its efficiency. The second remarkable property of AMP is that its optimality with respect to the Bayesian posterior \cref{eq:def:posterior} can be exactly characterised in the high-dimensional regime.  Indeed, \cite{Aubin2018} ---generalizing the earlier results of \cite{Barbier2019} on single index models--- proved the following  (formally, their proof was for the committee machine, but it applies {\it mutatis mutandis} to any multi-index models):
\begin{theorem}[Bayes-optimal correlation, Theorem 3.1 in \cite{Aubin2018}, informal] 
\label{thm:bo}

Let $(x_{i},y_{i})_{i\in[n]}$ denote $n$ i.i.d. samples from the multi-index model defined in \ref{def:model}. Denote by $\hat{\bW}_{\rm bo} = \mathbb{E}[W|X,y]\in\mathbb{R}^{p\times d}$ the mean of the posterior marginals \cref{eq:def:posterior}. 
Then, under Assumption \ref{ass:link} in the high-dimensional asymptotic limit where $n,d\to\infty$ with fixed ratio $\alpha = \sfrac{n}{d}$, the asymptotic correlation between the posterior mean and $\bW^{\star}$:
\begin{align}
    \bMM^{\star} = \lim\limits_{d\to\infty}\mathbb{E}\left[\frac{1}{d}\hat{\bW}_{\rm bo}{\bW^{\star}}^{\top}\right]
\end{align}
is the solution of the following $\sup\inf$ problem:
\begin{align}
\label{eq:freeenergy}
    \underset{\hat{\bMM}\in\mathcal{S}^{+}_{p}}{\sup}\underset{M\in\mathcal{S}^{+}_{p}}{\inf}\left\{-\frac{1}{2}\Tr \bMM\hat{\bMM}-\frac{1}{2}\log(\bI_{p}+\hat{\bMM})+\frac{1}{2}\hat{\bMM}+\alpha H_{Y}(\bMM)\right\}
\end{align}
where $H_{Y}(\bMM)=\mathbb{E}_{\bxi\sim\mathcal{N}(0,\bI_{p})}[H_{Y}(\bM|\bxi)]$, with $H_{Y}(\bMM|\bxi)$ the the conditional entropy of the effective $p$-dimensional estimation problem \cref{eq:effective}.
\end{theorem}
Note that the state evolution \cref{eq:replica_equation} is closely related to the $\sup\inf$ problem in \cref{eq:freeenergy}. Indeed, remarking that the update function $F$ in \cref{eq:replica_equation} is precisely the gradient of the entropy $H_{Y}$ in \cref{eq:effective}, one can show that state evolution is equivalent to gradient descent in the objective defined by \cref{eq:freeenergy} \citep{Aubin2018}. This non-trivial fact implies that whenever \cref{eq:freeenergy} has a single minima, AMP {\it optimally estimates} the posterior marginals, at least if its state evolution converges to this fixed point.

Finally, note that by construction $\hat{\bW}_{\rm bo}$ is the optimal estimator of $\bW^{\star}$ given the data $(X,y)$ (in the MMSE sense). Therefore, the rank of $M^{\star}$ defines the dimension of the statistically optimal subspace reconstruction at sample complexity $\alpha\coloneqq\sfrac{n}{d}$. The fact that AMP follows state evolution is proven for such problems in \citep{javanmard2013state,Gerbelot}. Specifically, we refer to the example 4.4 in \cite{Gerbelot} on Matrix-valued random variables.

\section{Bayes-AMP in the presence of side-information}\label{sec:side-inf}
The algorithm is easily generalized to the case with {\it side information} of the type eq.(\ref{ref:side_info}):
\begin{equation}
    {\bf S} = \sqrt{\lambda} W^\star +\sqrt{1-\lambda}Z, 
\end{equation}
The effect of the additional side-information is equivalent to a modification of the prior on $W$ to $P^S_W(W)$, defined as:

\begin{align*}
  P^S_W(\bW) &\propto P(\bS|\bW)P_W(\bW)={\cal N}(\sqrt{\lambda} {\bf S},\sqrt{1-\lambda}\mathbf{I})
\end{align*}

This is a simple consequence of Bayes theorem:

\begin{equation}
    P(\bW|\by,\bX,\bS) \propto P(\by,\bX,\bS|\bW)P_W(\bW)= P(\by,\bX|\bW)P(\bS|\bW)P_W(\bW)
\end{equation}

In turns, this change updates of  $\hat{\bW}$ and $\hat{\bC}$ in Algorithm \ref{alg:gamp} as described in \cite{Aubin2018}:

\begin{equation}
    \hat{\bW}_{k}^{t+1} = (\bI_{p}(1-\lambda) + \bA_{k}^{t})^{-1}(\bb_{k}^{t}(1-\lambda) + \sqrt{\lambda}\bA_{k}^{t}\bS_{k})\,,\qquad \hat{\bC}_{k}^{t+1} =  (\bI_{p}(1-\lambda)+\bA_{k}^{t})^{-1}\,.
\end{equation}

It is easily checked that the limit $\lambda \to 0$ gives back the equations in Algorithm \ref{alg:gamp}. Lemma \ref{thm:se} is then modified with $G(X)$ in equation \eqref{eq:replica_equation} being replaced by
\begin{equation}\label{eq:g_side_inf}
        G(X) = (X(1-\lambda)+\lambda) (1+X(1-\lambda))^{-1}\,.
\end{equation}

\section{Trivially-coupled Subspaces}\label{app:triv_coup}

In this section, we discuss and illustrate examples of $g$ where $T^\star_U \neq \emptyset$ for a learned subspace $U$. Our prototypical example throughout the discussion will be:
\begin{equation}
    g(z_1,z_2,z_3) = z^2_1+z^2_2+\text{sign}(z_1z_2z_3).
\end{equation}
Let $e_1,e_2,e_3$ denote the directions corresponding to components $z_1,z_2,z_3$ respectively.
It is easy to check that for the above example, $E^\star=\{e_1,e_2\}$ while $H^\star_E = \{e_3\}$.
Therefore, by Theorem \ref{thm:stab_init}, for large enough $\alpha$ the $\bMM^{t} \geq \delta M_E^\star$ where $ M_E^\star$ spans $z_1,z_2$ after $t =\mathcal{O}(\log \frac{1}{\lambda})$ iterations. Subsequently, since $T^\star_{E^\star} = \{e_3\}$, it gains an overlap along $e_3$ (independent of $\lambda$).
Unlike sequential learning of directions in $E^\star_U$, however, the time-steps for recovery of $\{e_1,e_2\}$ and $\{e_3\}$ (upto a threshold independent of $\lambda$) different only by a constant. The recovery of $e_3$ can therefore be interpreted as instantaneous w.r.t $\{e_1,e_2\}$. 

\section{Proofs of the main results}

In this section, we describe the proofs of our main results, starting with a \emph{Trivial} proof of the learning of Trivial subspaces (Theorem \ref{thm:triv}) and the form of linearization that will be utilized throughout the analysis.

\subsection{Theorem \ref{thm:triv} and Linearization}

Lemma \ref{thm:unin:existence} and Theorem \ref{thm:triv} are direct consequences of the following observation:

\begin{proposition}
    Let $\vec{v} \in \mathbb{R}^p$ be arbitrary.  Then, starting from $\bMM^0=0$, $\vec{v}^\top \bMM^1 \vec{v}>0$ if and only if:
\begin{equation}\label{eq:vspan}
\alpha\Ea{\vec{v}^\top\bg_{\rm out}\!\left(Y^{t},0,\bI_{p}\right)\bg_{\rm out}\!\left(Y^{t},0,\bI_{p}\right)^{\top}\vec{v}} >0
    \end{equation}
\end{proposition}

\begin{proof}
    The above proposition follows directly from the form of $F$ in Lemma \ref{thm:se} and the observation that the mapping $G(X)$ preserves the span of $X$.
\end{proof}
The expression in Equation \ref{eq:vspan} can be rewritten as:
\begin{equation}
\Ea{(\vec{v}^\top\bg_{\rm out}\!\left(Y^{t},0,\bI_{p}\right))^2}
\end{equation}
Since the operand inside the expectation is non-negative, the expectation is non-zero if and only if the operand vanishes identically. We therefore obtain Lemma \ref{thm:unin:existence} and Theorem \ref{thm:triv}.
 
\subsubsection{Linearization without side-information}
Our analysis relies on the following result:
\begin{lemma}\label{lem:lin_app}
Let $F(\bMM)$ be as defined in Lemma \ref{thm:se}
    \begin{equation}
   F(\bMM) \approx \alpha\mathcal{F}(\delta\bMM) + \mathcal{O}(\alpha\norm{\delta\bMM}_F^2),
\end{equation}
where $\norm{}_F$ denotes the Frobenius norm and $\mathcal{F}(\delta\bMM)$ is a linear operator on the cone $\mathcal{S}^+_p$ of PSD matrices of dimension $p$:
\begin{equation} 
    \mathcal{F}(\bMM) \defeq \,\mathbb{E}_{y}\left[ \partial_{\bomega} \bg_{\rm out}(y,0,\bI_p)\, \bMM\partial_{\bomega} \bg_{\rm out}(y,0,\bI_p)^\top  \right]\,,
\end{equation}
\end{lemma}
We proceed through an-entry-wise expansion of each term inside the expection in $F(\bMM)$ around $\bMM = 0$. Recall that $\bg_{\rm out} \in \mathcal{C}^2$ by Assumption \ref{ass:link}.
Since $\bMM_F \leq p$ , the first two derivatives of $\bg_{\rm out}$ are uniformly bounded in $\bg_{\rm out}$ for any fixed $Y, \bxi$.
Therefore, applying the multivariate Taylor expansion to $\bg_{\rm out}(Y^{t},\Delta_1\bxi, I+\Delta_2)$, with $\Delta_1 =\sqrt{M}$ and $\Delta_2= ,\bI_{p}-\bMM$ yields:  
\begin{align*}
    &\bg_{\rm out}\!\left(Y^{t},\sqrt{\bMM^{t}}\bxi,\bI_{p}-\bMM^{t}\right)\bg_{\rm out}\!\left(Y^{t},\sqrt{\bMM^{t}}\bxi,\bI_{p}-\bMM^{t}\right)^{\top}\\&=  \bg_{\rm out}\!\left(Y^{t},0,\bI_{p}-\bMM^{t}\right)\bg_{\rm out}\!\left(Y^{t},0,\bI_{p}-\bMM^{t}\right)^{\top} 
    +  \partial_{\bomega}\bg_{\rm out}\!\left(Y^{t},0,\bI_{p}\right)\sqrt{M}\bxi\bxi^\top\sqrt{M}^\top\partial_{\bomega}\bg_{\rm out}\!\left(Y^{t},0,\bI_{p}\right)^{\top}\\
    &+ \langle\partial_{\bV}\bg_{\rm out}\!\left(Y^{t},0,\bI_{p}\right)M\rangle\bg_{\rm out}\!\left(Y^{t},0,\bI_{p}\right)^{\top} + \bg_{\rm out}\!\left(Y^{t},0,\bI_{p}\right)\langle M,\partial_{\bV}\bg_{\rm out}\!\left(Y^{t},0,\bI_{p}\right)\rangle+\alpha C(\bxi,y)\mathcal{O}(\norm{M}^2_F),
\end{align*}
where $C(\bxi, Y)$ is an integrable function in $\bxi, Y$. 
Since $T^\star$ is empty, $\bg_{\rm out}\!\left(Y^{t},0,\bI_{p}\right)$ vanishes almost surely over $y$.
Therefore, using dominated-convergence theorem, we obtain: 
\begin{align*}
    F(\bMM) &= \alpha\Ea{\bg_{\rm out}\!\left(Y^{t},\sqrt{\bMM^{t}}\bxi,\bI_{p}-\bMM^{t}\right)\bg_{\rm out}\!\left(Y^{t},\sqrt{\bMM^{t}}\bxi,\bI_{p}-\bMM^{t}\right)^{\top}}\\ &= \alpha\Ea{\partial_{\bomega}\bg_{\rm out}\!\left(Y^{t},0,\bI_{p}\right)\sqrt{M}\bxi\bxi^\top\sqrt{M}^\top\partial_{\bomega}\bg_{\rm out}\!\left(Y^{t},0,\bI_{p}\right)^{\top}} + \alpha \mathcal{O}(\norm{M}^2_F)
\end{align*}

Since at $\bMM^{t} = 0$,  $\bxi$ and $Y$ are independent, the above simplifies to:
\begin{align*}
    F(\bMM)  &= \alpha\,\mathbb{E}_{y}\left[ \partial_{\bomega} \bg_{\rm out}(y,0,\bI_p-\bMM)\, \bMM\partial_{\bomega} \bg_{\rm out}(y,0,\bI_p-\bMM)^\top  \right]+ \alpha \mathcal{O}(\norm{M}^2_F)
\end{align*}

\subsubsection{Linearization in the presence of side-information}

Analogous to Lemma \ref{lem:lin_app}, we can linearize state-evolution in the presence of side-information in the limit of small $M,\lambda$.
\begin{lemma}\label{lem:side-info}
Let  $\mathcal{F}(\delta\bMM)$ be a linear operator on the cone $\mathcal{S}^+_p$ of PSD matrices of dimension $p$, defined by:
\begin{equation} 
    \mathcal{F}(\bMM) \defeq \,\mathbb{E}_{y}\left[ \partial_{\bomega} \bg_{\rm out}(y,0,\bI_p)\, \bMM\partial_{\bomega} \bg_{\rm out}(y,0,\bI_p)^\top  \right] 
\end{equation}
Then the following approximation holds:
\begin{equation}
   F(\bMM) \approx \alpha\mathcal{F}(\delta\bMM) + \sqrt{\lambda} \mathbf{I}_d + \mathcal{O}(\alpha\norm{\delta\bMM}_F^2) + \mathcal{O}(\lambda),
\end{equation}
\end{lemma}
\begin{proof}
    The above linearization is a direct consequence of Lemma \ref{lem:lin_app} along with the updated form of $G(X)$ in Equation \ref{eq:g_side_inf}.
\end{proof}

\subsection{Proof of Lemma \ref{lem:opt_t}}

\begin{proof}
    This is a direct consequence of the Tower law of expectation: $\Ea{\mathcal{T}(y) \langle \vec{v}^\top W^\star, \vec{x}\rangle} =  \Eb{y}{\mathcal{T}(y) \Eb{\vec{x}}{\langle \vec{v}^\top W^\star, \vec{x}\rangle|y }}$

The statement then follows by noting that $\bg_{\rm out}(y, \bzero, \bI_{p})^\top \vec{v}=  \lim_{d \rightarrow \infty}\Eb{\bx}{\langle \vec{v}^\top W^\star,\vec{x}\rangle|y}$.  The converse for any $v \in T^\star$ then follows by setting $\mathcal{T}(y)=\bg_{\rm out}(y, \bzero, \bI_{p})$.\end{proof}
\subsection{Proof of Lemma \ref{lem:stab_0}}\label{appendix:proof_stability}

Suppose that $M \in \mathcal{S}^+$, we have, for any $\vec{v} \in \R^p$:
\begin{equation} 
    \vec{v}^\top\mathcal{F}(\bMM)\vec{v}= \alpha\,\mathbb{E}_{y}\left[ \vec{v}^\top\partial_{\bomega} \bg_{\rm out}(y,0,\bI_p)\, \bMM\partial_{\bomega} \bg_{\rm out}(y,0,\bI_p)\vec{v} \right]\,,
\end{equation}
since $\bMM \in \mathcal{S}^+$, each term inside the expectation is non-negative. Therefore:
\begin{equation}
    \vec{v}^\top\mathcal{F}(\bMM)\vec{v} \geq 0
\end{equation}
Thus, $\mathcal{F}(M)$ is a cone-preserving linear map. 

The generalised Perron-Frobenius theorem/Krein-Rutman Theorem for cone-preserving maps \citep{krein1948linear,du2006order} then implies that the operator $\mathcal{F}(\bMM)$ admits at-least one eigenvector $M^* \in \mathcal{S}^+_p$  
corresponding to the largest eigenvalue $\nu_\mathcal{F}$ such that for any $M \in \mathcal{S}^+_p $:
\begin{equation}
    \mathcal{F}(\bMM) \leq \nu_\mathcal{F}\norm{\bMM}_F.
\end{equation}

Furthermore, all other eigenvalues are strictly smaller than $\lambda_\mathcal{F}$. Subsequently, Lemma \ref{lem:lin_app} implies that $ F(\bMM)$ is stable at $\bMM=0$ if and only if $\alpha \leq \frac{1}{\nu_\mathcal{F}}= \alpha_c$.

\subsection{Proof of Theorem \ref{thm:stab_init}}

Lemma \ref{thm:se} allows us to map the behavior of the variable $\bMM^{t}$ in state-evolution to high-probability statements for the limiting overlaps. 

Applying Lemmas \ref{lem:side-info} and \ref{lem:opt_t}, we obtain that for small enough $\bMM^{t+1},\lambda$:
\begin{equation}\label{eq:upp_bound}
    \norm{\bMM^{t+1}}_F \leq \alpha \nu_{\mathcal{F}}  \norm{\bMM^{t}}_F +\lambda+ C_1\norm{\bMM^{t}}^2_F+C_2\norm{\bMM^{t}}^2_F
\end{equation}
for some constants $C_1, C_2$. Now, suppose that $\alpha < \frac{1}{\nu_F}$. Then for $\norm{\bMM^{t}}_F  \geq  \frac{1}{(1-\alpha \nu_{\mathcal{F}})}\lambda$, we have:
\begin{equation}
    \alpha \nu_{\mathcal{F}}  \norm{\bMM^{t}}_F +\lambda \leq \norm{\bMM^{t}}_F.
\end{equation}

For small enough $\lambda$,
we have that:
\begin{equation}
    \bMM^{0} \approx \lambda < \frac{1}{(1-\alpha \nu_{\mathcal{F}})}\lambda.
\end{equation}

Subsequently, by induction, we obtain that:
\begin{equation}
    \sup_{t \geq p} \norm{\bMM^{t+1}}_F \leq \frac{1}{(1-\alpha \nu_{\mathcal{F}})}\lambda
\end{equation}
for all $t >0$. This proves the first part of the Theorem.

Now, suppose that $\alpha > \alpha_c$ or equivalently $\alpha \nu_F > 1$ and let $\bMM^\star \in \mathcal{S}^+_p$ denote any of the eigenvectors of $\mathcal{F}$ with eigenvalue $\nu_F$, i.e achieving the 
supremum in Equation \ref{eq:alpha_c}.

From Lemma \ref{lem:stab_0} and \ref{lem:side-info}, we obtain:
\begin{equation}\label{eq:fix_pt}
    \operatorname{tr}(\bMM^{t+1},\bMM^\star) \geq \alpha \nu_{\mathcal{F}}  \operatorname{tr}(\bMM^{t},\bMM^\star)+\lambda-  C_1\norm{\bMM^{t}}^2_F-C_2\lambda^2,
\end{equation}
for some constants $C_1,C_2>0$.

Since $\alpha\nu_F >1$,  we obtain that $\operatorname{tr}(\bMM^{1},\bMM^\star) > \lambda > 0$ for small enough $\lambda$.

Let $0<\kappa <\nu_F-1 $ be fixed. Equation \ref{eq:fix_pt} implies that for $\frac{\operatorname{tr}(\bMM^{t},\bMM^\star)}{\norm{\bMM^{t}}^2_F} < \frac{\alpha \nu_F-1-\kappa}{C_1}$, and small enough $\lambda$:
 
\begin{equation}\label{eq:growth}
    \operatorname{tr}(\bMM^{t+1},\bMM^\star) \geq (1+\kappa) \operatorname{tr}(\bMM^{t},\bMM^\star),
\end{equation}
implying that $\operatorname{tr}(\bMM^{t+1},\bMM^\star)$ grows as $\omega(e^{\kappa t})$. Since at $t=1$, $\frac{\operatorname{tr}(\bMM^{t},\bMM^\star)}{\norm{\bMM^{t}}^2_F} = \mathcal{O}(\frac{1}{\lambda})$, Equation \ref{eq:upp_bound} implies that the condition $\frac{\operatorname{tr}(\bMM^{t},\bMM^\star)}{\norm{\bMM^{t}}^2_F} < \frac{\alpha \nu_F-1-\kappa}{C_1}$ holds for time $\mathcal{O}(\log \frac{1}{\lambda})$, which along with Equation \ref{eq:growth} allows recovery upto an overlap $\delta$ independent of $\lambda$.

Finally, it remains to show that for large enough $\alpha$, $\bMM^{t}$ spans $E^\star$.

Notice that for any $\vec{v} \in E^\star$:
\begin{equation}
     \vec{v}^\top\mathcal{F}(\vec{v}\vec{v}^\top)\vec{v} = \mathbb{E}_{y}\left[  (\vec{v}^\top\partial_{\bomega} \bg_{\rm out}(y,0,\bI_p)\vec{v})^2\right] > 0,
\end{equation}
 
Define the smallest eigenvalue of $\mathcal{F}$:
\begin{equation}
    \nu^\star_\mathcal{F} = \inf_{\vec{v}\in E^\star, \norm{v}=1} \vec{v}^\top\mathcal{F}(\vec{v}\vec{v}^\top)\vec{v}.
\end{equation}
Since $\vec{v}$ lies in a compact set, $\nu^\star_\mathcal{F} > 0$.

Let $\vec{v}_1, \cdots, \vec{v}_k$ denote an orthonormal basis for $E^\star$
\begin{equation}  \vec{v}^\top_i(\bMM^{t+1})\vec{v}_i \geq \alpha \nu_\mathcal{F}\vec{v}^\top_i(\bMM^{t})\vec{v}_i + \lambda - \alpha C\norm{\bMM^{t}}^2_F.
\end{equation}
Therefore, for $\alpha \geq \frac{1}{\nu_\mathcal{F}}$ and small enough $\epsilon$, $\bMM^{t}$ expands linearly along each $\vec{v}_i$.

\subsection{Proof of Lemma \ref{lemma:He2}}

From the definition of $g_{out}$ in Equation \ref{eq:out_denoiser} and the dominated-convergence theorem, it is straightforward to see that:
\begin{equation}
     \partial_{\bomega}\bg_{\rm out}(Y,\bomega=\vec{0},\bI) = \mathbb{E}[\bZ\bZ^\top-\bI|Y].
\end{equation}
Therefore for any $\vec{v} \in \R^p$ with $\norm{\vec{v}}=1$:
\begin{align*}
\vec{v}^\top\partial_{\bomega}\bg_{\rm out}(Y,\bomega,\bI)_{\bomega=\vec{0}}\vec{v} &= \mathbb{E}[(\vec{v}^\top \bZ)^2-1|Y]\\
&=\lim_{d \rightarrow \infty}\mathbb{E}[\rm He_2(\langle \vec{v}^\top \bW^\star\vec{x}\rangle)|y=Y],
\end{align*}
where the last expectation is w.r.t the joint measure $p(\vec{x},y)$ defined in Definition \ref{def:model}.

Now, for any  measurable transformation $\mathcal{T}:\R \rightarrow \R$, we have, by the tower law of expectation:
 \begin{equation}
    \Ea{\mathcal{T}(y) {\rm He}_2(\langle \vec{v}^\top \bW^\star, \vec{x}\rangle)}= \Eb{y}{\mathcal{T}(y) \Eb{\vec{x}}{\rm He_2(\langle \vec{v}^\top \bW^\star, \vec{x}\rangle)|y}}.
 \end{equation}
For any $\vec{v} \in H^\star_E$, the integrand vanishes asymptotically almost surely over $y$. Therefore, the dominated-convergence theorem implies that for  $\vec{v} \in H^\star_E$:
\begin{equation}
   \lim_{d \rightarrow \infty}  \Ea{\mathcal{T}(y) {\rm He}_2(\langle \vec{v}^\top \bW^\star, \vec{x}\rangle)} = 0.
\end{equation}

\subsection{Staircase Dynamics and the proof of Proposition \ref{prop:hid}}\label{app:hid}

Before proving Proposition \ref{prop:hid}, we start by proving a monotonicity property of the AMP iterates, which in-turn implies that the AMP iterates gaining overlap along a certain subspace must eventually reach an approximate fixed point.

\subsubsection{Monotonicity of the AMP Iterates}

We begin by showing that the overlaps of the Bayes-AMP iterates is non-decreasing in the following sense:
\begin{lemma}\label{lem:monot}
    Consider the AMP algorithm with the Bayes-optimal choice of $f_t,g_t$. Then for any $t>0$:
    \begin{equation}
        \bMM^{(t+1)} \succeq \bMM^{(t)}
    \end{equation}
\end{lemma}
\begin{proof}
Our proof relies on a characterization of the state-evolution for the Bayes-AMP iterates as the solution to certain low-dimensional inference problems.
See for instance \cite{Barbier2019,Aubin2018}. Concretely, consider the following sequence of denoising problems over variables $Z^\star \in \mathbb{R}^k, W^\star \in \mathbb{R}^k$:
\begin{align*}
Y_t=g(\sqrt{\bMM^{(t)}}\xi+\sqrt{\mathbb{I}_k-\bMM^{(t)}}Z^\star)\\ 
    \tilde{Y}_t = \sqrt{\tilde{M}}_t W^\star + Z 
\end{align*}
where $Z, \xi \sim \mathcal{N}(0,\mathbf{I}_k)$ and
$\tilde{M}_t$ is obtained as:
\begin{equation}
\begin{split}
    \tilde{\bMM}^{(t)} & \defeq \Ea{(\Ea{Z^\star|Y,\xi}\Ea{Z^\star|Y,\xi}^\top)}\\
&=\Ea{g_{\text{out}}(Y,\sqrt{{\bMM^{(t)}}}\xi)g_{\text{out}}(Y,\sqrt{{\bMM^{(t)}}}\xi)^\top}
\end{split}
\end{equation}
We claim that the following statements hold for any $t>0$:
\begin{equation}
    \begin{split}
        \bMM^{(t+1)} \succeq \bMM^{(t)}\\
        \tilde{\bMM}^{(t+1)} \succeq \tilde{\bMM}^{(t)}
    \end{split}
\end{equation}
We proceed by induction. At $t=0$, $\bMM^{(t)}=0$ and thus $\bMM^{(t+1)} \succeq \bMM^{(t)}$ holds trivially.
Subsequently, we show via an information-theoretic argument that $\bMM^{(t+1)} \succeq \bMM^{(t)}$ implies $\tilde{\bMM}^{(t+1)} \succeq \tilde{\bMM}^{(t)}$. Suppose that $\bMM^{(t+1)} \succeq \bMM^{(t)}$ holds at time-step $t$. We decompose the field entering the channel for $Y_{t+1}$ as follows:
\begin{align*}
Y_{t+1} &= g(\sqrt{\bMM^{(t+1)}}\xi+\sqrt{\mathbb{I}_k-\bMM^{(t+1)}}Z^\star)\\
 &= g(\sqrt{\bMM^{(t)}}\xi_1+ \sqrt{\bMM^{(t+1)}-\bMM^{(t)}}\xi_2+ \sqrt{\mathbb{I}_k-\bMM^{(t+1)}}Z^\star),
\end{align*}
where $\xi_1,\xi_2$ denote independent variables $\sim \mathcal{N}(0,1)$. Now, compare the following two estimates:
\begin{equation}
\hat{Z}_1=\Ea{Z^\star|Y,\xi_1,\xi_2}, \quad \hat{Z}_2=\Ea{Z^\star|Y,\xi_1}.
\end{equation}
Note that by linearity of expectation and independence of $\xi, Z^\star$, the estimator $\hat{Z}_2$ equals the estimator of the inference problem at time $t$:
\begin{equation}
Y_t=g(\sqrt{\bMM^{(t)}}\xi+\sqrt{\mathbb{I}_k-\bMM^{(t)}}Z^\star)
\end{equation}
By the optimality of the posterior-mean w.r.t MSE error, we must have for any $v \in \mathbb{R}^k$ with $\norm{v}=1$
\begin{equation}\label{eq:mse}
    \Ea{(\langle v,\hat{Z}_1\rangle-\langle v,\ Z^\star\rangle)^2} \leq \Ea{(\langle v,\hat{Z}_2\rangle-\langle v,\ Z^\star\rangle)^2} 
\end{equation},
where $\langle v,\hat{Z}_i\rangle$ for $i=1,2$ denotes the projection of $\hat{Z}$ along $v$ or equivalently the Bayes-estimator of $\langle v,\ Z^\star\rangle$. Now, a straightforward consequence of Bayes-rule (Nishimori's identity) implies that:
\begin{equation}
    \Ea{\langle v,\hat{Z}_i\rangle \langle v,\ Z^\star\rangle}=  \Ea{(\langle v,\hat{Z}_i\rangle )^2},
\end{equation}
for $i=1,2$. Substituting in Equation \ref{eq:mse}, we obtain:
\begin{equation}
    \Ea{\langle v,\hat{Z}_1\rangle \langle v,\ Z^\star\rangle} \geq  \Ea{\langle v,\hat{Z}_2\rangle \langle v,\ Z^\star\rangle}.
\end{equation}
The above terms can be equivalently expressed as $v^\top \tilde{\bMM}^{(t+1)}v,v^\top \tilde{\bMM}^{(t)}v$ respectively. We obtain:
\begin{equation}
    \tilde{\bMM}^{(t+1)} \succeq \tilde{\bMM}^{(t)}.
\end{equation}

Similarly, we have that 
 $\tilde{\bMM}^{(t+1)} \succeq \tilde{\bMM}^{(t)}$ implies  $\bMM^{(t+2)} \succeq \bMM^{(t+1)}$. The claim therefore follows by induction. 

\end{proof}

\subsubsection{Sequential visits to Fixed-points}

Now, suppose that $E^\star \neq \phi$ and consider the dynamics with side-information provided only along $E^\star$ i.e:
\begin{equation}
    S = \sqrt{\lambda}P_{E^\star}+\sqrt{1-\lambda}\mathbb{I}, +
\end{equation}
where $P_{E^\star}$ denotes the projection along $E^\star$. Suppose further tha $\alpha > \alpha_c$. By Theorem \ref{thm:stab_init}, $\bMM^{(t)} \succ \delta M_{E^\star}$ in $\mathcal{O}(\log \frac{1}{\lambda})$ iterations. Since $G(X) \prec \mathbf{I}_k$ for any $X \in \mathcal{S}^+$, we have:
\begin{equation}
    \bMM^{(t)} \prec \mathbf{I}_k,
\end{equation}
for all $t > 0$. Considering the topology induced by the Frobenius norm on $\mathbb{R}^{p \times p}$, we obtain that $\bMM^{(t)}$ is a bounded sequence and therefore contains a convergent subsequence. By Lemma \ref{lem:monot}, we obtain the the convergence of the entire sequence to a fixed point $\bMM^f_{E^\star}(\lambda,\alpha)$.

Therefore, for any $\epsilon > 0$, $\exists t'$ such that for all $t > t'(\epsilon,\lambda)$:
\begin{equation}\label{eq:cauchy}
    \norm{\bMM^{(t)}-\bMM^f_{E^\star}(\lambda,\alpha)} \leq \epsilon.
\end{equation}
Next, Lemma \ref{lem:monot} further implies that:
\begin{equation}
    \bMM^f_{E^\star}(\lambda,\alpha) \prec  \bMM^f_{E^\star}(\lambda',\alpha) 
\end{equation},
for $\lambda < \lambda'$. Thus the following limit exists:
\begin{equation}
  \bMM^f_{E^\star}(\alpha) = \lim_{\lambda \rightarrow 0^+}  \bMM^f_{E^\star}(\lambda,\alpha).
\end{equation}

We next show that for small enough $\lambda$, even with access to full side-information, the AMP iterates ``visit" $\bMM^f_{E^\star}$

\begin{proposition}\label{prop:saddle}
    Suppose that  $E^\star \neq \phi$, $T^\star=\phi$ and $T^\star_{E^\star}=\phi$, and $\alpha > \alpha_c$. Then under side-information $S = \sqrt{\lambda}W^\star+\sqrt{1-\lambda}\mathbb{I}$, for any $\epsilon $, for any $\epsilon > 0$ and small enough $\lambda$, there exists $t'(\epsilon,\lambda)$ such that for  
    \begin{equation}
\norm{\bMM^{(t'(\epsilon,\lambda))}-\bMM^f_{E^\star}(\alpha)} \leq \epsilon.
    \end{equation}
\end{proposition}
\begin{proof}
Let $U=E^\star$ and consider the linearized dynamics along the orthogonal complement $(U)_\perp$:
    \begin{equation}
    \mathcal{F}_{\bMM_U}(\bMM_\perp) \defeq (\mathbf{I}+\bMM_U)^{-1}\mathbb{E}_{y}\left[ \partial_{\bomega} \bg_{\rm out}(y,\sqrt{\bMM_U}\bxi,\bI-\sqrt{M}_{U})\bMM_\perp\partial_{\bomega} \bg_{\rm out}(y,\sqrt{\bMM_U}\bxi,\bI-\sqrt{M}_{U})^\top  \right]\,,
\end{equation}
where $\bMM_U \subseteq E^\star$. By assumption, $\mathcal{F}_{\bMM_U}(\bMM_\perp)=0$ for $\bMM_U=0$. Therefore,
by continuity of $\mathcal{F}_{\bMM_U}(\bMM_\perp)$ w.r.t $\bMM_U$, for any $\alpha>\alpha_c$, $\exists \delta > 0$ such that if $\bMM_U \prec \delta \mathbf{I}_p$, $\mathcal{F}_{\bMM_U}(\bMM_\perp)$ remains contractive. For small enough $\delta$, we may therefore split the dynamics into two phases:
\begin{enumerate}
    \item Phase 1: $0 < t < t_\delta$ where $t_\delta= \inf_{t >0} \bMM^t_U \succ \delta \mathbf{I}_p$.
    \item Phase 2: $t > t_\delta$. 
\end{enumerate}
From Theorem \ref{thm:AMP_random}, we know that $t_\delta = \mathcal{O}(\delta \log \frac{1}{\lambda})$. By the choice of $\delta$, we have by the end of phase $1$,
$\bMM_\perp \prec C \lambda \mathbb{I}$ for some constant $C_1$. By linearization \ref{lem:side-info}, we obtain that  for $t$ additional steps in Phase 2, the overlap along $\bMM_\perp$ remains bounded by $e^{K t \lambda}$ for some constant $K>0$. We can further ensure that for small enough $\delta$ we can further ensure that $\bMM^{t_\delta}_U \prec C_2\delta \mathbf{I}_p$, for some constant $C_2$.
Subsequently, by considering the limit $\delta \rightarrow 0$ with $\lambda << \delta$, and the roles of $\lambda, \delta$ interchanged, we obtain that after $\mathcal{O}(t'(\epsilon,\delta))$ additional steps, where $t'(\epsilon,\delta)$ is defined as in Equation \ref{eq:cauchy}:
\begin{equation}
\norm{\bMM^{(t'(\epsilon,\delta))}-\bMM^f_{E^\star}(\alpha)} \leq \epsilon
\end{equation},
since $t'(\epsilon,\delta)$ is independent of $\lambda$, the overlap along $\bMM_\perp$ at $t'(\epsilon,\delta)$ remains $\mathcal{O}(\lambda)$.
\end{proof}

The above proposition showed that, by restricting the side-information to specific subspaces, we can uniquely characterize sequences of fixed points along sequences of growing subspaces. 

We're now ready to state the detailed version of Proposition \ref{prop:hid} along with its proof, which states that upon the inclusion of additional side-information along a new subspace:
\begin{proposition}[Full-version of Proposition \ref{prop:hid}]
Let $U \subseteq \mathbb{R}^p$ be a subspace such that $E^\star_U$ is non-empty.. Consider AMP iterates with the Bayes-optimal choice of $f_t,g_t$ with side-information restricted to the subspace $U$ for sufficiently small $\lambda>0$. Suppose that $\bMM_U$ is a fixed-point under such a dynamics where $\bMM_U$ spans  $U$. 
Then, under the dynamics with an additional side-information along the subspace $E^\star_U$:
\vspace{-0.4cm}
\begin{itemize}[leftmargin=*, noitemsep,wide=1pt]
    \item For $\alpha>\alpha_{\text{gst}}(\bMM_{U})$, AMP first reaches arbitrarily close to $\bMM_U$ in  $\mathcal{O}(\log \frac{1}{\lambda})$ iterations and subsequently recovers $M^*_{M_U}$ in additional $\mathcal{O}(\log \frac{1}{\lambda})$  steps for arbitrarily small $\lambda$, where   $M^*_{M_U}$  denotes any matrix in $\mathcal{S}^+_p$ achieving the supremum $\sup_{\bMM^\perp\in U^\perp} \norm{\mathcal{F}_{\bMM_U}(\bMM^\perp)}_F$ as per Definition (\ref{def:denoise_coup}).
    \item For $\alpha<\alpha_{\text{gst}}(\bMM_{U})$ and sufficiently small $\lambda$, AMP reaches 
  arbitrarily close to $\bMM_U$ in  $\mathcal{O}(\log \frac{1}{\lambda})$ and the overlap along $E^\star_U$ remains bounded by $C\lambda$.
\end{itemize}
\end{proposition}
\begin{proof}
    By assumption, $M_U$ is a fixed point of $F(M)$ under the presence of side-information restricted to $U$ Therefore, Equation \ref{eq:replica_equation} implies that $\bg_{\rm out}(y,\bomega,\bI-\sqrt{M}_{U})$ almost surely lies in $U$ and thus $T^\star_U$ is empty. Therefore, by the proof of Proposition \ref{prop:saddle}, for sufficiently small $\lambda$, and any $\epsilon > 0$ AMP iterates reach $\epsilon$-close to $\bMM_U$ in time $\mathcal{O}(\log\frac{1}{\lambda})$ iterations, while the overlap along the orthogonal complement to $\bMM_{U}$ being $\mathcal{O}(\lambda)$. Theorem \ref{prop:hid} then follows that of \ref{thm:stab_init} by considering the following linearized operator along $U^\perp$
\begin{equation}
    \mathcal{F}_{\bMM_U}(\bMM_\perp) \defeq (\mathbf{I}+\bMM_U)^{-1}\mathbb{E}_{y}\left[ \partial_{\bomega} \bg_{\rm out}(y,\sqrt{\bMM_U}\bxi,\bI-\sqrt{M}_{U})\bMM_\perp\partial_{\bomega} \bg_{\rm out}(y,\sqrt{\bMM_U}\bxi,\bI-\sqrt{M}_{U})^\top  \right]\,,
\end{equation}
where $\bomega=\sqrt{M}_{U}\bxi$ with $\bxi \sim \mathcal{N}(0, \bI)$. Similar to Lemma \ref{lem:lin_app}, the linearization follows by noting that $y$ is independent of $\sqrt{\bMM_\perp}\bxi$ and $\sqrt{\bMM_U+\bMM_\perp} = \sqrt{\bMM_U}+\sqrt{\bMM_\perp}$ since $\bMM_\perp \perp \bMM_U$. 
\end{proof}

\subsection{Proof of Theorem \ref{thm:staircase}}

Theorem \ref{thm:staircase} follows by inductively applying Propositions \ref{prop:hid} and \ref{prop:saddle} as we describe below:
\begin{enumerate}
    \item First, suppose that $T^\star = \emptyset$ and $T^\star_{E^\star_k} = \emptyset$ for all $k \in \mathbb{N}$.
    By compactness, for any $k$, $\inf_{\vec{v}\in E^\star_{E^\star_k}} v^\top \mathcal{F}_{\bMM_U}(\bMM_\perp) \vec{v} > 0$. Therefore, as in the proof of Theorem \ref{thm:AMP_random}, for each $E^\star_k$,  $\exists \alpha_k$  such that for $\alpha > \alpha_k$, the corresponding fixed-point along $E^\star_k$ becomes unstable along the full subspace $E^\star_{E^\star_k}$, leading to weak-recovery of $E^\star_{k+1}$ in $\mathcal{O}(\log \frac{1}{\lambda})$ additional steps.
    Since $k$ is bounded above by $p$, we obtain that $\alpha^\star=\max_k \alpha_k$ is finite. For $\alpha>\alpha^\star$, Proposition \ref{prop:saddle} implies that AMP iterates sequentially visit $E^\star_k$ and eventually recover $E^\star_{gst}$.
    \item In the presence of trivial subspaces $T^\star$, as shown in the proof of Theorem \ref{thm:triv} and explained in section \ref{app:triv_coup}, the trivially-coupled subspaces are learned instantaneously. Therefore, the sequence of fixed-points traversed by AMP span the following sequence of subspaces:
    \begin{enumerate}
        \item $U^\star_1 = T^{\star\star}_1$.
        \item $U^\star_{j+1} = U^\star_j \cup E^\star_{U^\star_j} \cup T^{\star\star}_{E^\star_{U^\star_j}}$, for $j=1, 2, \cdots$,
    \end{enumerate}
    where $T^{\star\star}_{U}$ denotes the maximal chain of trivial subspaces starting from $U$:
    \begin{equation}
       T^{\star\star}_{U} =  T^{\star}_{U} \cup T^{\star}_{T^{\star}_{U}} \cup \cdots. 
    \end{equation}
    It is easy to check that the set subspaces $U \in \mathbb{R}^p$ is totally ordered under the span of $U\cup T^\star_U \cup E^\star_U$ and therefore possesses a unique maximal element. Therefore, both sequences $U^\star_i$ defined above and $E^\star_i$ defined in Definition \ref{def:stair} terminate at the same maximal element $E^\star_{gst}$.
    Similar to $(i)$, proposition \ref{prop:saddle} then implies that for large enough $\alpha$, and small enough $\lambda$ AMP traverses the sequence of subspaces $U^\star_1, U^\star_2, \cdots$, recovering  $E^\star_{gst}$ in $\mathcal{O}(\log(\frac{1}{\lambda}))$ iterations.
\end{enumerate}

\section{From AMP thresholds to Computational phase transitions for First-order algorithms}\label{app:comp_transitions}

In this section, we discuss how Theorems \ref{thm:stab_init} and \ref{thm:staircase} can  be translated to rigorous lower-bounds on the sample complexity of weak-recovery from arbitrarily small initialization. For brevity, we omit the full-proof and provide a sketch of the argument basewd on existing results.

We recall the main results of \citep{celentano20a}. \citep{celentano20a} consider the following class of general first-order algorithms (GFOM):

\begin{align}
    \bv^{t} &= \bX f^{(1)}_t(\bu^1, \cdots,\bu^t,\by, \bu) + f^{(2)}_{t}(\bv^1, \cdots,\bv^t,{\bf y})\\
    {\bf u}^{t+1} &= \bX^T g^{(1)}_t(\bv^1, \cdots,\bv^t,\bv) + g^{(2)}_{t} (\bu^1, \cdots,\bu^t,\by,\bu)  ,     
\end{align}
where $\bv^{t} \in \mathbb{R}^{n \times r},\bu^{t}\in \mathbb{R}^{n \times r}$. Note that Algorithm \ref{alg:gamp} is  a special case of the above class. The variables $\bu,\bv$ encode additional side-information.

Subsequently, \cite{celentano20a} show that any algorithm belonging to the above class can be mapped to a message-passing algorithm on an infinite-tree such that the  errors of the GFOM algorithm at any finite-time $t$ asymptotically match the errors of the surrogate estimation problem on the tree. For the tree-inference problem, the optimality of the Bayes-AMP directly follows from AMP being the dense-limit of the information-theoretically optimal Belief propagation on trees. This translates to the optimality of Bayes-optimal AMP at any finite-time $t$ amongst GFOM algorithms.

Similar to our Theorem \ref{thm:AMP_random}, \cite{celentano20a} provide thresholds $\alpha_c$ for weak-recovery starting from arbitrarily small initialization in the case of sparse Phase retrieval and sparse PCA. As we discussed in Section \ref{sec:side-inf}, the Bayes-optimal AMP in the presence of such side-information corresponds to a modified prior.

While \cite{celentano20a} prove their results for the specific classes of high-dimensional regression (single-index) problems and low-rank matrix estimation, the underlying argument is general enough to be adapted to our setting of multi-index models. We therefore conjecture that the thresholds and class of problems described in Theorems \ref{thm:triv},\ref{thm:AMP_random},\ref{thm:staircase} fundamentally characterize the learnability of multi-index models through first-order algorithms.

\section{Explicit results for some examples of the \texorpdfstring{$\bg_{\rm out}$ function}{}}
\label{app:examples}

In this section, we list several problems, their corresponding $\bg_{\rm out}$, and analysis of their weak recoverability. 
For convenience we will use $\mathcal{Z}_{\rm out}(y, \bomega, \bV)$ defined as

\begin{equation}\label{app:def:Z_out}
    \mathcal{Z}_{\rm out}(y, \bomega, \bV) =
    \frac{1}{(2\pi)^{p/2}}\int_{\mathbb{R}^{p}} \dd \bz \, e^{-\frac{1}{2}(\bz-\bomega)^{\top}\bV^{-1}(\bz-\bomega)}P(y|\bz)
\end{equation}
which allows us to rewrite $\bg_{\rm out}(y,\bomega,\bV)$
\begin{equation}
    \bg_{\rm out}(y,\bomega,\bV) =  
    \frac{1}{\mathcal{Z}_{\rm out}(y, \bomega, \bV)} \frac{1}{(2\pi)^{p/2}}\int_{\mathbb{R}^{p}} \dd \bz \,e^{-\frac{1}{2}(\bz-\bomega)^{\top}\bV^{-1}(\bz-\bomega)}P(y|\bz)\bV^{-1}(\bz-\bomega)
\end{equation}
When computing the critical sample complexity we will also need $\partial_\bomega \bg_{\rm out}(y, \bomega, \bV)$, which can be expressed conveniently as
\begin{align*}
    &\partial_\bomega \bg_{\rm out}(y, \bomega, \bV)=\\&\frac{1}{\mathcal{Z}_{\rm out}(y, \bomega, \bV)} \frac{1}{(2\pi)^{p/2}}\int_{\mathbb{R}^{p}} \dd \bz \,e^{-\frac{1}{2}(\bz-\bomega)^{\top}\bV^{-1}(\bz-\bomega)}(\bz-\bomega) P(y|\bz)\bV^{-2}(\bz-\bomega)^\intercal - \bV^{-2} - \bg_{\rm out}(y, \bomega, \bI_p)\bg_{\rm out}(y, \bomega, \bI_p)^\intercal
\end{align*}
which becomes much simpler for $\bV = \bI_p$, $\bomega = \bzero$, and $\bg_{\rm out}(y, \bzero, \bI_p) = \bzero$:
\begin{align*}
    \partial_\bomega \bg_{\rm out}(y, \bzero, \bI_p)=\frac{1}{\mathcal{Z}_{\rm out}(y, \bzero, \bI_p)} \frac{1}{(2\pi)^{p/2}}\int_{\mathbb{R}^{p}} \dd \bz \,e^{-\frac{1}{2}\bz^{\top}\bz}(\bz\bz^\intercal - \bI_p) P(y|\bz)
\end{align*}.

\subsection{\texorpdfstring{$g(z_1, z_2, ...., z_p) = \prod_{j=1}^p z_j$}{}}
In this subsection we will compute the critical sample complexity for this class of models. The first observation we can make is that for all $p$, $\bg_{\rm out}(y, \bzero, \bI_p)=0$.
We can look at the cases $p=2$ and $p \geq 3$ separately. 

If $p=2$ we have:
\begin{equation}
    \mathcal{Z}_{\rm out}(y, \bzero, \bI_p) = \frac{K_0(|y|)}{\pi}
\end{equation}
and 
\begin{equation}
    \partial_\bomega \bg_{\rm out}(y, \bzero, \bI_p) = \begin{bmatrix}
        y\frac{K_1(y)}{K_0(y)}-1 & y\\
        y & y\frac{K_1(y)}{K_0(y)}-1\\
    \end{bmatrix}
\end{equation}
where $K_n(y)$ are the modified Bessel functions of the second kind. The integrals can be easily done in Mathematica. Since our model is invariant under permutations all the elements in the diagonal and out of diagonal need to take the same value. 

If $p>3$ the picture changes completely. First, we first have
\begin{equation}
    \mathcal{Z}_{\rm out}(y, \bzero, \bI_p) = \frac{1}{(2\pi)^{p/2}}G^{p, 0}_{0, p} \left( \frac{y^2}{2^p} \, \bigg| \, \begin{array}{c}
0 \\
0, 0, \ldots, 0
\end{array} \right)
\end{equation}
The next, crucial observation is that the out of diagonal terms of $\partial_\bzero \bg_{\rm out}(y, \bzero, \bI_p)$ are zero.
This is again a consequence of symmetry. We can see it in the practical example $p=3$. From the explicit equation at the beginning of the appendix we see that out of diagonal elements of $\partial_\bzero \bg_{\rm out}(y, \bzero, \bI_p)$ are proportional to the value of the following integral:
\begin{equation}
    \frac{1}{(2\pi)^{3/2}}\int_{\mathbb{R}^{3}} \dd \bz \,e^{-\frac{z_1^2 + z_2^2 + z_3^2}{2}} z_1 z_2 \,\delta(y - z_1 z_2 z_3)
\end{equation}
This integral is zero, as we can do a change of variable $\{z_1 \to z_1,\, z_2\to -z_2,\, z_3 \to -z_3 \}$: the delta function and the Gaussian measure are left unchanged, while $z_1 z_2$ switches sign. A similar approach can be used to show that out of diagonal elements are zero for every $p$.

The diagonal elements of $\partial_\bzero \bg_{\rm out}(y, \bzero, \bI_p)$ take all the same value $f(y)$ because of permutation invariance of the argument of $g$, even though there is not a simple expression of them. We propose the following writing based on Meyer's G function
\begin{equation}
    f(y) = 2 G^{p, 0}_{0, p} \left( \frac{y^2}{2^p} \, \bigg| \, \begin{array}{c}
0 \\
0, 0, \ldots, 0, 1
\end{array} \right) / G^{p, 0}_{0, p} \left( \frac{y^2}{2^p} \, \bigg| \, \begin{array}{c}
0 \\
0, 0, \ldots, 0
\end{array} \right) - 1
\end{equation}
We are finally able to compute the critical sample complexity $\alpha_c$. Unfortunately in the case $p=2$ there is no immediate expression for $\alpha_c$. Yet, it's possible to plug the expressions we just derived into \eqref{eq:stab_theorem} and perform the optimization in \eqref{eq:alpha_c} numerically. The case $p \geq 3$ is much easier, as $\partial_\bzero \bg_{\rm out}(y, \bzero, \bI_p) = f(y) \bI_p$
We first start with \eqref{eq:stab_theorem}
\begin{equation}
    \mathcal{F}(\bMM) = \mathbb{E}_{Y}\left[ \partial_{\bomega} \bg_{\rm out}(Y,\bzero,\bI_p)\, \bMM\partial_{\bomega} \bg_{\rm out}(Y,\bzero,\bI_p)^\top  \right] = \mathbb{E}_{Y}\left[ f(Y)^2  \right] M.
\end{equation}
The alpha critical $\alpha_c$ will thus be
\begin{equation}
    \alpha_c = \mathbb{E}_{Y}\left[ f(y)^2  \right]^{-1} = \left[\int_{-\infty}^{\infty}\,{\rm d}y f(y)^2  \mathcal{Z}_{\rm out}(y)\right]^{-1}
\end{equation}
Performing the integral numerically suggests the scaling law in the main text.

\subsection{\texorpdfstring{$\frac{1}{p}\sum_{i=1}^p z_i^2$}{}}
We are only interested in computing the critical sample complexity for the case $p>1$. The function $g$ is even, so $\bg_{\rm out}(y, \bzero, \bI_p)=0$. The trick for the computation is to move to generalised spherical coordinates by choosing $r^2 = z_2^2 + ... z_p^2$. Recall that the area of the unit sphere in $p-1$ dimensions is 
\begin{equation}
    \frac{2 \pi^{\frac{p-1}{2}}}{\Gamma\left(\frac{p-1}{2}\right)}
\end{equation}
The matrix $\partial_\bomega \bg_{\rm out}(y, \bzero, \bI_p)$ is proportional to the identity. This is because we have permutation invariance with respect to the arguments in $g$, so the matrix $\partial_\bomega \bg_{\rm out}(y, \bzero, \bI_p)$ has the same value on the diagonal and the same off diagonal. The off diagonal elements are proportional to 
\begin{equation}
    \frac{1}{(2\pi)^{p/2}}\int_{\mathbb{R}^{p}} \dd \bz \,e^{-\frac{1}{2}\sum_{i=1}^p z_i^2} z_1 z_2 \,\delta\left(y - \frac{1}{p}\sum_{i=1}^p z_p^2\right)
\end{equation}
If we do a change of variable $z_1 \to -z_1$ and all the other variables the same the integral changes sign, so this integral has to be zero. This shows that the matrix is diagonal, and because of the invariance under permutation also proportional to the identity. We are left to compute the actual value on the diagonal. First, we compute $\mathcal{Z}_{\rm out}$:
\begin{equation}
    \mathcal{Z}_{\rm out}(y, \bzero, \bI_p) = \frac{1}{(2\pi)^{p/2}}\int_{\mathbb{R}^{p}} \dd \bz \,e^{-\frac{1}{2}\sum_{i=1}^p z_i^2}  \,\delta\left(y - \frac{1}{p}\sum_{i=1}^p z_p^2\right)
\end{equation}
The integrand can be rewritten using the generalized spherical coordinates on the last $p-1$ variables
\begin{equation}
    \frac{1}{(2\pi)^{p/2}}\int_{\mathbb{R}^{p}} \dd \bz \, e^{-\frac{1}{2}\sum_{i=1}^p z_i^2}  \,\delta\left(y - \frac{1}{p}\sum_{i=1}^p z_p^2\right) = \frac{1}{2^{p/2-1}\Gamma\left(\frac{p-1}{2}\right)\sqrt{\pi}} \int_0^{+\infty} \dd r \int_{\mathbb{R}} \dd z_1  \,r^{p-2} \, e^{-\frac{z_1^2 + r^2}{2}}  \,\delta\left(y - \frac{z_1^2 + r^2}{p}\right)
\end{equation}
The last two integrals can be done by hand. Integrating over $z_1$ and then over $r$ we get
\begin{equation}
    \mathcal{Z}_{\rm out}(y, \bzero, \bI_p) = \frac{1}{2^{p/2-1}\Gamma\left(\frac{p-1}{2}\right)\sqrt{\pi}} \int_0^{\sqrt{p y}} \dd r  \,p\, r^{p-2} \, \frac{e^{-\frac{p y}{2}}}{\sqrt{p y - r^2}} = \frac{e^{-\frac{p y}{2}}}{y 2^{p/2}\Gamma\left(\frac{p}{2}\right)}  \,(p y)^{p/2}
\end{equation}
One can proceed in the same fashion to show that
\begin{equation}
    \frac{1}{(2\pi)^{p/2}}\int_{\mathbb{R}^{p}} \dd \bz \,e^{-\frac{1}{2}\sum_{i=1}^p z_i^2} (z_1^2-1) \,\delta\left(y - \frac{1}{p}\sum_{i=1}^p z_i^2\right) = \frac{e^{-\frac{p y}{2}}}{y 2^{p/2}\Gamma\left(\frac{p}{2}\right)}  \,(p y)^{p/2} (y - 1)\,,
\end{equation}
which gives us the result
\begin{equation}
    \partial_\bomega \bg_{\rm out}(y, \bzero, \bI_p) = (y-1)\bI_p\,.
\end{equation}
We can finally look at \eqref{eq:stab_theorem} and obtain
\begin{equation}
    \mathcal{F}(\bMM) = \mathbb{E}_{Y}\left[ (Y - 1)^2 \right] \bMM = \int_0^{+\infty} \dd y  \,\frac{e^{-\frac{p y}{2}}}{y 2^{p/2}\Gamma\left(\frac{p}{2}\right)}  \,(p y)^{p/2} (y-1)^2 \, \bMM = \frac{2}{p} \bMM
\end{equation}
from which it's immediate to obtain $\alpha_c = p/2$.

\subsection{\texorpdfstring{$g(z_1,\ldots,z_p) = \operatorname{sign}(\prod_{j=1}^pz_j)$}{}}\label{app:gout:prod_of_signs}
In this part of the section we will compute $\bg_{\rm out}$ as in eq. (\ref{eq:out_denoiser}) for arbitrary values of the argument.
Define $\bomega = (\omega_1, \ldots, \omega_p)^\top$, $v_{ij}\coloneqq(\bV^{-1})_{i,j}$ and the Gaussian probability density
\begin{equation}
    \rho(\bz\in\R^p,\bV ) := \frac{1}{(2\pi)^{\nicefrac{p}{2}}\sqrt{\det(\bV)}}e^{-\frac{1}{2}\bz^\top\bV^{-1}\bz}.
\end{equation}
Then, defining $\R_{+1} = [0, +\infty)$ and $\R_{-1} = (-\infty, 0]$, we have that eq. (\ref{app:def:Z_out})
\begin{align}
    \mathcal{Z}_{\rm out}(y, \bomega, \bV) &= \int_{\R^p}{\rm d}z_1\ldots{\rm d}z_p \rho(\bz-\bomega,\bV)\delta_{y,\operatorname{sign}(z_1\ldots z_p)}\\
    &=\sum_{s_1,\ldots,s_p=\pm1}\delta_{y,(\prod_j s_j)}\int_{\R_{s_1}}{\rm d}z_1\ldots\int_{\R_{s_p}}{\rm d}z_p \rho(\bz-\bomega,\bV).
\end{align}
In order to simplify this expression, we introduce the auxiliary function
\begin{align*}
     I_s(z_1,\ldots,z_{p-1},y,\bomega,\bV) &\coloneqq \int_{\R_s}{\rm d}z_p \frac{e^{-(\bz-\bomega)^T\bV^{-1}(\bz-\bomega)}}{(2\pi)^{\sfrac{p}{2}}\sqrt{\det \bV}} \\
    &= \frac{e^{-\sum_{i,j=1}^{p-1}(z_i-\omega_i)v_{ij}(z_j-\omega_j)}}{(2\pi)^{\sfrac{p}{2}}\sqrt{\det \bV}}\int_{\R_s}{\rm d}z_p \exp\left(-v_{pp}(z_p-\omega_p)^2 - 2(z_p-\omega_p)\sum_{i=1}^{p-1}v_{ip}(z_i-\omega_i)\right)\\
    &=\frac{e^{-\sum_{i,j=1}^{p-1}(z_i-\omega_i)(v_{ij}-v_{pp}^{-1}v_{ip}v_{jp})(z_j-\omega_j)}}{2v_{pp}\sqrt{2}(2\pi)^{\sfrac{(p-1)}{2}}\sqrt{\det \bV}}\left(1-s\cdot\operatorname{erf}\left(\frac{\sum_{i=1}^{p-1}v_{ip}(z_i-\omega_i) - v_{pp}\omega_p}{\sqrt{v_{pp}}}\right)\right)
\end{align*}
We introduce the matrix $\boldsymbol{d}(\bV)\in\mathbb{R}^{(p-1)\times(p-1)}$ such that
\begin{align}
    d_{ij}(\bV) &:= \det\left(\begin{array}{cc}
      v_{ij}   & v_{pi} \\
       v_{jp}  & v_{pp}
    \end{array}\right), \quad i,j \in{1,\ldots,p-1},
\end{align}
and the function
\begin{equation*}
    E_s(\bz\in\R^{p-1}, y, \bomega,\bV) = \frac{1}{2\sqrt{2}v_{pp}}\left(1-s\cdot\operatorname{erf}\left(\frac{\sum_{i=1}^{p-1}v_{ip}(z_i-\omega_i) - v_{pp}\omega_p}{\sqrt{v_{pp}}}\right)\right).
\end{equation*}

We can rewrite\footnote{It can be show that $\det(\boldsymbol{d}(\bV)) = \det(\bV^{-1})$ by considering the defining properties of the determinant.}
\begin{align}
    I_{s}(\bz\in\R^{p-1}, y, \bomega, \bV) &:= E_s(\bz, y, \bomega,\bV)\frac{\exp\left(-v_{pp}^{-1}\sum_{i,j=1}^{p-1}(z_i-\omega_i)d_{ij}(\bV)(z_j-\omega_j)\right)}{(2\pi)^{\sfrac{(p-1)}{2}}\sqrt{\det (\boldsymbol{d}(\bV)^{-1})}}\\
    &=v_{pp}^{\sfrac{(p-1)}{2}}E_s(\bz, y, \bomega,\bV)
    \rho\left(\left(\begin{array}{ccc}
        z_1 - \omega_2\\\ldots\\  z_{p-1}-\omega_{p-1}\end{array}\right),v_{pp}\boldsymbol{d}(\bV)^{-1}\right).
\end{align}
Therefore
\begin{align}
    \mathcal{Z}_{\rm out}(y,\bomega\bV) &= \sum_{s_1,\ldots,s_p = \pm1}\delta_{y, (\prod_js_j)}\int_{\R_{s_1}}{\rm d}z_1\ldots\int_{\R_{s_{p-1}}}{\rm d}z_{p-1} I_{s_p}(\bz, y, \bomega,\bV)\\
    &=\sum_{s_1,\ldots,s_{p-1} = \pm1}\int_{\R_{s_1}}{\rm d}z_1\ldots\int_{\R_{s_{p-1}}}{\rm d}z_{p-1} I_{ys_1\ldots s_{p-1}}(\bz, y, \bomega,\bV).
\end{align}
The components of $\bg_{\rm out}(y,\bomega,\bV)$ are given by
\begin{align*}
    (\bg_{\rm out}(y,\bomega,\bV))_p &=\frac{1}{  \mathcal{Z}_{\rm out}(y,\bomega\bV)} \sum_{s_1,\ldots,s_{p-1} = \pm1}\int_{\R_{s_1}}{\rm d}z_1\ldots\int_{\R_{s_{p-1}}}{\rm d}z_{p-1} \frac{\partial}{\partial \omega_p} I_{ys_1\ldots s_{p-1}}(z_1,\ldots,z_{p-1}, y, \bomega,\bV)\\
    &= \frac{1}{  \mathcal{Z}_{\rm out}(y,\bomega\bV)} \sum_{s_1,\ldots,s_{p-1} = \pm1}y\prod_{j=1}^{p-1}s_j\int_{\R_{s_1}}{\rm d}z_1\ldots\int_{\R_{s_{p-1}}}{\rm d}z_{p-1} \rho\left(\left(\begin{array}{ccc}
        z_1 - \omega_1\\\ldots\\  z_{p-1}-\omega_{p-1}\\-\omega_p\end{array}\right),\bV\right)\\
        &=\frac{1}{\mathcal{Z}_{\rm out}(y,\bomega\bV)} \sum_{s_1,\ldots,s_{p} = \pm1}y\prod_{j\neq p}s_j\int_{\R_{s_1}}{\rm d}z_1\ldots\int_{\R_{s_{p}}}{\rm d}z_{p} \rho\left(\bz - \bomega,\bV\right)\delta(z_p),
\end{align*}
and, analogously,
\begin{equation*}
    (\bg_{\rm out}(y,\bomega,\bV))_k = \frac{1}{\mathcal{Z}_{\rm out}(y,\bomega\bV)} \sum_{s_1,\ldots,s_{p} = \pm1}y\prod_{j\neq k}s_j\int_{\R_{s_1}}{\rm d}z_1\ldots\int_{\R_{s_{p}}}{\rm d}z_{p} \rho\left(\bz - \bomega,\bV\right)\delta(z_k).
\end{equation*}
\paragraph{Stability of the fixed point and critical sample complexity} \label{sec:app:parity}
We can now focus on studying the critical sample complexity. The first step is to compute $\partial_{\bomega}\bg_{\rm out}(y,\bzero,\bI_p)$ by using the identity
\begin{equation}
    \partial_{\bomega}\bg_{\rm out}(y,\bzero,\bI) = \frac{1}{Z_{\rm out}(y,\bzero,\bI)}\int{\rm d}\bz \frac{e^{-\frac{{\bz}^\intercal\bz}{2}}}{(2\pi)^{p/2}}\left(\bz\bz^\intercal - \bI\right)\delta\left(y - g(\bz)\right)
\end{equation}

To this end we know that 
\begin{equation}
    \mathcal{Z}_{\rm out}(y,\bzero,\bI) = \frac{1}{(2\pi)^{p/2}}\int {\rm d}\bz e^{-\frac{\bz^\top\bz}{2}} \delta_{y, \operatorname{sign}(z_1\ldots z_p)}= \frac{1}{2}
\end{equation}

In order to compute $\partial_\omega \bg_{\rm out}(y,\bzero,\bI)$, we note at first that
\begin{equation}
  \frac{1}{(2\pi)^{p/2}}\int {\rm d}\bz (z_k^2-1)e^{-\nicefrac{\bz^\top\bz}{2}} \delta_{y, \operatorname{sign}(z_1\ldots z_p)} = 0,\quad\forall k\in\{1,\ldots,p\}.
\end{equation}
For $p=2$, consider the integral
\begin{align}
    &\frac{1}{2\pi}\int{\rm d}z_1{\rm d}z_2\, z_1z_2 e^{-\frac{z_1^2}{2} - \frac{z_2^2}{2}}\delta_{y, \operatorname{sign}(z_1\ldots z_p)} \\
    =& \frac{2}{2\pi}\left(\delta_{y,1}\int_0^\infty{\rm d}z_1\int_0^\infty{\rm d}z_2\, z_1z_2 e^{-\frac{z_1^2}{2} - \frac{z_2^2}{2}} + \delta_{y,-1}\int_0^\infty{\rm d}z_1\int_{-\infty}^0{\rm d}z_2\, z_1z_2 e^{-\frac{z_1^2}{2} - \frac{z_2^2}{2}}\right)\\
    &=\frac{{\rm sign}{(y)}}{\pi}
\end{align}
Thus we have that 
\begin{equation}
    \partial_{\bomega}\bg_{\rm out}(y,\bzero,\bI) = \frac{2\,{\rm sign}(y)}{\pi}
    \begin{bmatrix}
        0 & 1\\
        1 & 0\\
    \end{bmatrix}
\end{equation}
Now we can notice that \eqref{eq:alpha_c} simplifies significantly as $\partial_{\bomega}\bg_{\rm out}(y,\bzero,\bI)$ is proportional to the counterdiagonal matrix for all values of $y$, thus having that
\begin{align*}
    \mathcal{F}(\bMM) \defeq \mathbb{E}_{Y}\left[ \partial_{\bomega} \bg_{\rm out}(Y,\bzero,\bI_p)\, \bMM\partial_{\bomega} \bg_{\rm out}(Y,\bzero,\bI_p)^\top  \right] = \frac{4}{\pi^2} \begin{bmatrix}
        0 & 1\\
        1 & 0\\
    \end{bmatrix}\bMM \begin{bmatrix}
        0 & 1\\
        1 & 0\\
    \end{bmatrix} 
\end{align*}
Finally we have that 
\begin{equation}
    \| \mathcal{F}(\bMM)\|_F = \frac{4}{\pi^2}\left\|\begin{bmatrix}
        0 & 1\\
        1 & 0\\
    \end{bmatrix}\bMM \begin{bmatrix}
        0 & 1\\
        1 & 0\\
    \end{bmatrix}\right\|_F = \frac{4}{\pi^2}\| \bMM\|_F
\end{equation}
Using \eqref{eq:alpha_c} we obtain that $\alpha_c$ is simply

$\alpha_c = \pi^2/4$.
For $p\geq 3$, defining $\R^p_{(\pm)}=\{\bz\in\R^p|\operatorname{sign}(z_1\ldots z_p)=\pm1\}$, we need to compute integrals of the type
\begin{align*}
    &\frac{1}{(2\pi)^{\nicefrac{p}{2}}}\int_{\R^p_{(+)}}{\rm d}z_1 \ldots{\rm d} z_p z_1 z_2 e^{-\frac{z_1^2}{2}}\ldots e^{-\frac{z_p^2}{2}} \\=& \frac{2}{2^{p-2}(2\pi)^{\nicefrac{p}{2}}}\left(\sum_{j=0}^{\lfloor\frac{p-2}{2}\rfloor}\left(\begin{array}{cc}
        p-2 \\
         2j
    \end{array}\right)\int_{\R^2_{(+)}} {\rm d}z_1{\rm d}z_2\, z_1z_2 e^{-\frac{z_1^2}{2} - \frac{z_2^2}{2}} +
    \sum_{j=1}^{\lfloor\frac{p-1}{2}\rfloor}\left(\begin{array}{cc}
        p-2 \\
         2j-1
    \end{array}\right)\int_{\R^2_{(-)}} {\rm d}z_1{\rm d}z_2\, z_1z_2 e^{-\frac{z_1^2}{2} - \frac{z_2^2}{2}}\right)\\
    =&\frac{2^{3-p}}{(2\pi)^{(p+2)/2}}\sum_{j=0}^{p-2}\left(\begin{array}{cc}
         p-2  \\
         j 
    \end{array}\right)(-1)^j = 0
\end{align*}
In the same way it is possible to prove that the integral over $\R_{(-)}^p$ is also vanishing. This implies that $\partial_{\bomega} \bg_{\rm out}=0$, so using \eqref{eq:alpha_c}, we obtain that $\alpha_c = +\infty$. This model cannot be learned with $n=\mathcal{O}(d)$ samples for $p \geq 3$.

\subsection{\texorpdfstring{$g(z_1,\ldots,z_p) = z_1^2 + \operatorname{sign}\left(\prod_{j=1}^pz_j\right)$}{}}

Following a procedure similar to Section \ref{app:gout:prod_of_signs}, we define $\bomega = (\omega_1, \ldots, \omega_p)^\top$ and $v_{ij}\coloneqq(\bV^{-1})_{i,j}$. 

We also introduce the matrix $\boldsymbol{d}(\bV)\in\mathbb{R}^{(p-1)\times(p-1)}$ such that
\begin{align}
    d_{ij}(\bV) &:= \det\left(\begin{array}{cc}
      v_{i,j}   & v_{i,p} \\
       v_{j,p}  & v_{p,p}
    \end{array}\right), \quad i,j \in{1,\ldots,p},
\end{align}
the function
\begin{equation*}
    E_{s,s_1,s_p}(\bz\in\R^{p-1}, y, \bomega,\bV) = \frac{1}{2\sqrt{2}v_{pp}}\left(1-s_p\cdot\operatorname{erf}\left(\frac{v_{1p}(s_1\sqrt{y-s}-\omega_1)+\sum_{i=2}^{p-1}v_{ip}(z_i-\omega_i) - v_{pp}\omega_p}{\sqrt{v_{pp}}}\right)\right),
\end{equation*}
the Gaussian probability density
\begin{equation}
    \rho(\bz\in\R^p,\bV ) := \frac{1}{(2\pi)^{\nicefrac{p}{2}}\sqrt{\det(\bV)}}e^{-\frac{1}{2}\bz^\top\bV^{-1}\bz},
\end{equation}
and finally
\begin{equation*}
    I_{s,s_1,s_{p-1}}(\bz\in\R^{p-1}, y, \bomega, \bV) := 
    \frac{E_{s,s_1,s_p}(\bz, y, \bomega,\bV)}{2\sqrt{y-s}}
    \rho\left(\left(\begin{array}{cccc}
        s_1\sqrt{y- s}-\omega_1  \\
        z_2 - \omega_2
        \ldots\\
        z_{p-1}-\omega_{p-1}  
    \end{array}\right),v_{pp}\boldsymbol{d}(\bV)^{-1}\right).
\end{equation*}
Then, defining $\R_{+1} = [0, +\infty)$ and $\R_{-1} = (-\infty, 0]$, we have that
\begin{align}\label{app:eq:phase_sign_Zout}
    \mathcal{Z}_{\rm out}(y,\bomega\bV) &= \sum_{s,s_1,\ldots,s_{p} = \pm1} \delta_{s,(\prod_js_j)}\boldsymbol{1}_{y > s}\int_{\R_{s_2}}{\rm d}z_2\ldots\int_{\R_{s_{p-1}}}{\rm d}z_{p-1} I_{s,s_1,s_p}(\bz, y, \bomega,\bV)\\
    &=\sum_{s,s_1,\ldots,s_{p-1} = \pm1} \boldsymbol{1}_{y > s}\int_{\R_{s_2}}{\rm d}z_2\ldots\int_{\R_{s_{p-1}}}{\rm d}z_{p-1} I_{s,s_1,s\cdot s_1...s_{p-1}}(\bz, y, \bomega,\bV),
\end{align}
where $\boldsymbol{1}_{x>0}$ is the Heaviside step function.
The components of $\bg_{\rm out}(y,\bomega,\bV)$ are given by
\begin{align*}
    (\bg_{\rm out}(y,\bomega,\bV))_1 &= \sum_{s,s_1,\ldots,s_{p-1} = \pm1} \frac{\boldsymbol{1}_{y > s}}{ \mathcal{Z}_{\rm out}(y,\bomega\bV) }\int_{\R_{s_2}}{\rm d}z_2\ldots\int_{\R_{s_{p-1}}}{\rm d}z_{p-1} \frac{\partial}{\partial \omega_1}I_{s,s_1,s\cdot s_1\ldots s_{p-1}}(\bz, y, \bomega,\bV)\\
    (\bg_{\rm out}(y,\bomega,\bV))_{k\neq1} &= \sum_{s,s_1,\ldots,s_p = \pm1} \frac{s\prod_{j\neq k}s_j\boldsymbol{1}_{y > s}}{ \mathcal{Z}_{\rm out}(y,\bomega\bV) }\int_{\R_{s_2}}{\rm d}z_2\ldots\int_{\R_{s_{p}}}{\rm d}z_{p}\;\rho\left(\left(\begin{array}{ccc}
        s_1\sqrt{y-s} - \omega_1\\z_2-\omega_2\\\ldots\\  z_{p}-\omega_{p}\end{array}\right),\bV\right)\delta(z_k)
\end{align*}
\paragraph{Stability of the fixed point} \label{app:gout:square_sign_alpha}
    We will use the same technique as \ref{app:gout:prod_of_signs}. The result is dependent on a number of integrals which we can compute analytically. First we have
    \begin{align}
        \mathcal{Z}_{\rm out}(y, \bzero, \bI) &= \frac{1}{(2\pi)^{\nicefrac{p}{2}}}\int{\rm d}z_1\ldots{\rm d}z_p \delta(y - z_1^2 - {\rm sign}(z_1\ldots z_p)) e^{-\frac{z_1^2}{2}-\ldots-\frac{z_p^2}{2}}\\
        &= \frac{1}{2(2\pi)^{\nicefrac{p}{2}}}\sum_{s_1,\ldots,s_p=\pm1} \frac{\boldsymbol{1}_{y > s}}{\sqrt{y-\prod_js_j}}\int_{\R_{s_2}}{\rm d}z_2\ldots\int_{\R_{s_{p-1}}}{\rm d}z_{p-1}e^{-\frac{y-s_1\ldots s_p}{2}-\frac{z_2^2}{2}-\ldots-\frac{z_p^2}{2}}\\
        &= \frac{1}{2\sqrt{2\pi}}\sum_{s=\pm 1} \boldsymbol{1}_{y > s}\frac{e^{-\frac{y-s}{2}}}{\sqrt{y-s}}\\
        &=\frac{e^{-\nicefrac{y}{2}}}{2\sqrt{2\pi e}}\left(\frac{1}{\sqrt{y+1}}+\bone_{y>1}\frac{e}{\sqrt{y-1}}\right)
    \end{align}

    From this it is straightforward to see that $\bg_{\rm out}^0(y)=0$. In order to compute $\nabla_\omega \bg_{\rm out}(y,0,1)$ we need to consider the additional integrals
    \begin{align}1)\;
        &\frac{1}{(2\pi)^{\nicefrac{p}{2}}}\int{\rm d}z_1\ldots{\rm d}z_p (z_1^2-1) e^{-\frac{z_1^2}{2}-\ldots-\frac{z_p^2}{2}}\delta(y - z_1^2 - {\rm sign}(z_1\ldots z_p))\\
        &= \frac{1}{2\sqrt{2\pi}}\sum_{s=\pm 1} \boldsymbol{1}_{y > s}\frac{e^{-\frac{y-s}{2}}(y-s-1)}{\sqrt{y-s}}\\
        &=\frac{e^{-y}}{2\sqrt{2\pi e}}\left(\frac{y}{\sqrt{y+1}}+\bone_{y>1}\frac{e(y-2)}{\sqrt{y-1}}\right)\\
        2)\;&\frac{1}{(2\pi)^{\nicefrac{p}{2}}}\int{\rm d}z_1\ldots{\rm d}z_p (z_{k\neq 1}^2-1) e^{-\frac{z_1^2}{2}-\ldots-\frac{z_p^2}{2}}\delta(y - z_1^2 - {\rm sign}(z_1\ldots z_p)) = 0\\
        3)\;&\frac{1}{(2\pi)^{\nicefrac{p}{2}}}\int{\rm d}z_1\ldots{\rm d}z_p z_jz_k e^{-\frac{z_1^2}{2}-\ldots-\frac{z_p^2}{2}}\delta(y - z_1^2 - {\rm sign}(z_1\ldots z_p)) \\
        &\propto \sum_{s,s_j,s_k =\pm1}\boldsymbol{1}_{y > s}s_js_ke^{-\frac{y-s}{2}} = 0\\
    \end{align}
This shows that
\begin{equation}
\mathcal{Z}_{\rm out}(y,\bzero, \bI_p)\partial_\bomega \bg_{\rm out}(y, \bzero, \bI_p) = \begin{bmatrix}
        C(y) & 0 & 0 \\
        0    & 0 & 0 \\
        0     & 0 & 0 \\
    \end{bmatrix}
\end{equation}
where $C(y)$ is
\begin{equation}
    C(y) = \begin{cases}
        y & -1<y<1 \\
        y -2e\frac{\sqrt{y+1}}{e\sqrt{y+1} + \sqrt{y-1}}-1 & y>1\\
    \end{cases}
\end{equation}
This means that
\begin{align*}
    \mathcal{F}(\bMM) \defeq \mathbb{E}_{Y}\left[ \partial_{\bomega} \bg_{\rm out}(Y,\bzero,\bI_p)\, \bMM\partial_{\bomega} \bg_{\rm out}(Y,\bzero,\bI_p)^\top  \right] = \int\,{\rm d}y C(y)^2 \mathcal{Z}_{\rm out}(y) \bMM_{11}
\end{align*}
The critical alpha is thus simply
\begin{equation}
    \alpha_c = \left[\int\,{\rm d}y C(y)^2 \mathcal{Z}_{\rm out}(y) \right]^{-1} \approx 0.575166
\end{equation}

\subsection{\texorpdfstring{$g(z_1, \ldots,z_p)= \sum_{j=1}^p\operatorname{sign}(z_p)$}{}}

\begin{figure}[t]
\begin{center}
\includegraphics[width=0.48\textwidth]{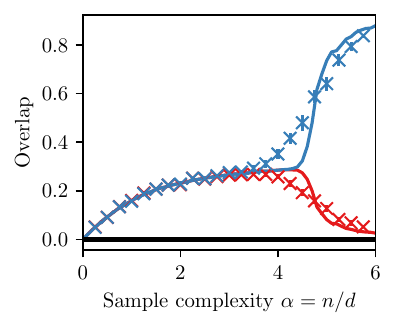}
\end{center}
    \caption{Specialisation as an example of grand staircase with $g\!=\!{\rm sign}(z_{1})\!+\!{\rm sign}(z_{2})\!+\!{\rm sign}(z_{3})$. The direction spanned by $z_1\!+\!z_2\!+\!z_3$ is first learned for any $\alpha>0$. The remaining ones are only learned at the {\it specialisation} transition occurring at $\alpha\!\approx\!4.3$. We take care of the symmetries as indicated in S.I. \ref{sec:app:numerics}. Crosses denote AMP runs with $d\!=\!100$ averaged over $72$ seeds. The overlap is shown as a function of the sample complexity. Because of the symmetry the elements of $M$ can take one of two values: one on the diagonal and one outside of it. We display $(M_{11}+M_{22}+M_{33})/3$ in {\color{blue} blue} and  $(M_{12}+M_{13}+M_{23})/3$ in {\color{red} red}.}
    \label{fig:committee}
\end{figure}

Following a procedure similar to Section \ref{app:gout:prod_of_signs}, we define $\bomega = (\omega_1, \ldots, \omega_p)^\top$ and $v_{ij}\coloneqq(\bV^{-1})_{i,j}$.

We also introduce the matrix $\boldsymbol{d}(\bV)\in\mathbb{R}^{(p-1)\times(p-1)}$ such that
\begin{align}
    d_{ij}(\bV) &:= \det\left(\begin{array}{cc}
      v_{i,j}   & v_{i,p} \\
       v_{j,p}  & v_{p,p}
    \end{array}\right), \quad i,j \in{1,\ldots,p},
\end{align}
the function
\begin{equation}
    E_s(\bz\in\R^{p-1}, y, \bomega,\bV) = \frac{1}{2\sqrt{2}v_{pp}}\left(1-s\cdot\operatorname{erf}\left(\frac{\sum_{i=1}^{p-1}v_{ip}(z_i-\omega_i) - v_{pp}\omega_p}{\sqrt{v_{pp}}}\right)\right),
\end{equation}
the Gaussian probability density
\begin{equation}
    \rho(\bz\in\R^p,\bV ) := \frac{1}{(2\pi)^{\nicefrac{p}{2}}\sqrt{\det(\bV)}}e^{-\frac{1}{2}\bz^\top\bV^{-1}\bz},
\end{equation}
and finally
\begin{equation*}
    I_{s}(\bz\in\R^{p-1}, y, \bomega, \bV) =v_{pp}^{\sfrac{(p-1)}{2}}E_s(\bz, y, \bomega,\bV)
    \rho\left(\left(\begin{array}{ccc}
        z_1 - \omega_1\\\ldots\\  z_{p-1}-\omega_{p-1}\end{array}\right),v_{pp}\boldsymbol{d}(\bV)^{-1}\right).
\end{equation*}
Then, defining $\R_{+1} = [0, +\infty)$ and $\R_{-1} = (-\infty, 0]$, we have that
\begin{align}
    \mathcal{Z}_{\rm out}=\sum_{\substack{{s_1,...,s_{p-1} = \pm 1 }\\{ |\sum_{i=1}^{p-1}s_i-y|=1}}}
    \int_{\R_{s_1}}{\rm d}z_{1}\ldots\int_{\R_{s_{p-1}}}{\rm d}z_{p-1}I_{y-\sum_{i=1}^{p-1}s_i}(\bz, y, \bomega, \bV),
\end{align}
and the $k^{\text{th}}$ component of $\bg_{\rm out}(y,\bomega,\bV)$ is
\begin{align*}
    \left(\bg_{\rm out}(y, \bomega,\bV)\right)_k = \frac{1}{\mathcal{Z}_{\rm out}}\sum_{\substack{{s_1,...,s_{p} = \pm 1 }\\{ |\sum_{i\neq k}s_i-y|=1}}}
    \left(y-\sum_{i\neq k}s_i\right)\int_{\R_{s_1}}{\rm d}z_{1}\ldots\int_{\R_{s_{p}}}{\rm d}z_{p}\rho(\bz-\bomega,\bV)\delta(z_k)
\end{align*}

\section{Further numerical observations}\label{sec:app:numerics}

\begin{figure*}
    \centering
    \includegraphics{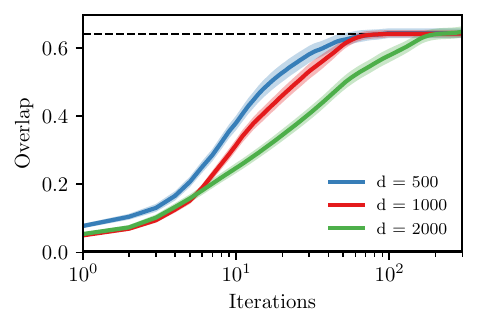}
    \caption{Trajectories of $16$ AMP runs with random initialisation at $\alpha = 4$ for $g(z_1,z_2) = {\rm sign}(z_1z_2)$. The shaded areas represent the error on the mean. We can see that as the dimension increases the algorithm gets increasingly slower}
    \label{fig:log_time}
\end{figure*}

Here we give more details about the numerical implementation of State Evolution and AMP. Both approaches require to compute $\bg_{\rm out}$. All the examples we implemented are detailed in S.I sec. \ref{app:examples}.

The integrals are doing with the quadrature package in Scipy. In order to avoid instabilities we regularize the interval of integration by replacing $\infty$ with $\Lambda$. We typically choose $\Lambda \approx 10$. Similarly, we add $\epsilon \approx 10^{-4}$ to the diagonal of $V$.
In State Evolution we need to integrate over functions of $\bg_{\rm out}$. We do these integrals using a simple Monte Carlo approach. For Figure \ref{fig:z1_sign} we used a total $72000$ samples, for Figure \ref{fig:committee} $7200$. 
Computing such integrals is the numerical bottleneck. In order to make this part faster we parallelised the MCMC: for each iteration of State Evolution we make every worker in our pool estimate the integral, and then average the estimation and then average among workers.
In the cases in which $\bMM=0$ is a fixed point, we initialize $\bMM$ with the empirical overlap of AMP at the beginning of the iteration.

In both the AMP and State Evolution implementation we used some damping: the overlap $\bMM$ or the $\hat{\bW}$ at the new iterations are averaged with the current value, with a weight $\delta$ for the new one, where typically $0.6 < \delta< 0.9$. We display the evolution of the overlaps in a typical run of AMP for the model $z_1^2 + {\rm sign}(z_1 z_2 z_3)$ in Figure \ref{fig:dynamics}. We already displayed the values of the overlap and generalisation error at convergence in Figure \ref{fig:z1_sign}. We can see how AMP has a saddle-to-saddle dynamic, where the algorithm alternates plateaus for $\mathcal{O}(\log{d})$ iterations to fast drops in generalisation error, which are associated with new directions being learned.
We probe experimentally the dependence on size in Figure \ref{fig:log_time}.
All the error bars present represent the error on the mean. We used the local computing cluster for our experiments. In particular we needed around $40$ CPUs with $72$ cores each and approximately $3$ Gb per core.

As stated in the main text, models non-trivial subspaces are associated with symmetries. This also implies that the associated overlaps are have the same invariances. In order to make the plots readable we remove all such symmetries by hand. For this reason we take the absolute value of the overlap if the model is even, and we impose a specific inequality in the overlap if there is invariance under permutation.
As a general idea we want to always have the "best" possible overlap. Meaning we want $\bMM$ to be as diagonal as possible. We list what this mean for the examples in the figures:
\begin{itemize}
    \item ${\rm sign} (z_1 z_2)$: Because of invariance under exchange of $z_1$ and $z_2$ we always have $\bMM_{11}=\bMM_{22}$. Here AMP can reach $2$ equivalent configurations: either the diagonal or the anti-diagonal is zero. We choose the configuration where the anti-diagonal is zero.
    \item $z_1^2 + {\rm sign}(z_1 z_2 z_3)$:
    This model will first learn just $z_1$. We fix $\bMM_{11}>0$. For the rest of the components we are reduced to the case above.
    \item ${\rm sign}(z_1) + {\rm sign}(z_2) + {\rm sign}(z_3)$: Because of the invariance under permutation each row of $\bMM$ will have either the same element in all the entries (in which case we don't need to do anything) or two are the same and one is bigger. We permute the matrix such that the largest entry is on the diagonal.
\end{itemize}

The code to run AMP and State Evolution on the examples in the  is available on GitHub

\url{https://github.com/SPOC-group/FundamentalLimitsMultiIndex}

\end{document}